\documentclass[lettersize,journal]{IEEEtran}
\usepackage{amsthm}
\newtheorem{defn}{\textbf{Definition}}
\newtheorem*{prob_state}{\textbf{Problem Statement}}

\newtheorem{proposition}{\textbf{Proposition}}

\usepackage{color}
\usepackage{graphicx}
\usepackage{subfig}
\usepackage{amsmath}
\usepackage{amssymb}
\usepackage{mathrsfs}
\usepackage[ruled, vlined, linesnumbered]{algorithm2e}
\usepackage[hidelinks]{hyperref}
\usepackage{cite}
\usepackage{makecell}
\usepackage[numbers]{natbib}
\usepackage{array}
\usepackage{tikz}
\usetikzlibrary{trees}
\usepackage{tabularx,colortbl}
\usepackage{threeparttable}
\usepackage{booktabs}
\usepackage{multirow}
\usepackage{longtable}
\usepackage{rotating}
\usepackage[capitalize,noabbrev]{cleveref}
\crefname{equation}{Eq.}{Eqs.}
\usepackage{ragged2e}
\usepackage{braket}

\usepackage{enumitem}
\usepackage{tcolorbox}
\usepackage{mdframed}
\usepackage{fancybox}

\usepackage{floatrow}
\usepackage{diagbox}

\usepackage{adjustbox}

\newcommand{\filledcircle}{\tikz\fill[black] (0,0) circle (.8ex);}
\newcommand{\emptycircle}{\tikz\draw (0,0) circle (.8ex);}

\definecolor{DeepPink}{HTML}{FF1493}
\definecolor{Orchid}{HTML}{DA70D6}
\definecolor{Magenta}{HTML}{FF00FF}
\definecolor{Fuchsia}{HTML}{FF00FF}
\definecolor{LavenderPink}{HTML}{FFB6C1}
\definecolor{verylightgray}{rgb}{0.9, 0.9, 0.9}
\definecolor{lightred}{rgb}{1,0.8,0.8}

\begin{document}
%

\title{SMS: Self-supervised Model Seeding for Verification of Machine Unlearning}

\author{Weiqi~Wang,~\IEEEmembership{Member,~IEEE},
	Chenhan~Zhang,~\IEEEmembership{Member,~IEEE},
        Zhiyi~Tian,~\IEEEmembership{Member,~IEEE}, \\
        and Shui~Yu,~\IEEEmembership{Fellow,~IEEE}
\IEEEcompsocitemizethanks{\IEEEcompsocthanksitem W. Wang, C. Zhang, Z. Tian and S. Yu are with the School of Computer Science, University of Technology Sydney, Australia.\protect\\
E-mail: {\{weiqi.wang, chenhan.zhang, zhiyi.tian-1, shui.yu\}@uts.edu.au}
}
\noindent
\thanks{\it{ This work is partially supported by Australia ARC LP220100453 and ARC DP240100955. (Corresponding author: Zhiyi Tian.)}}
}

%
%

\markboth{This paper has been accepted by IEEE Transactions on Dependable and Secure Computing}%
{Shell \MakeLowercase{\textit{et al.}}: Bare Advanced Demo of IEEEtran.cls for IEEE Computer Society Journals}
%

\IEEEtitleabstractindextext{%
\begin{abstract}	
\justifying 
Many machine unlearning methods have been proposed recently to uphold users' right to be forgotten. However, offering users verification of their data removal post-unlearning is an important yet under-explored problem. Current verifications typically rely on backdooring, i.e., adding backdoored samples to influence model performance. Nevertheless, the backdoor methods can merely establish a connection between backdoored samples and models but fail to connect the backdoor with genuine samples. Thus, the backdoor removal can only confirm the unlearning of backdoored samples, not users' genuine samples, as genuine samples are independent of backdoored ones. In this paper, we propose a Self-supervised Model Seeding (SMS) scheme to provide unlearning verification for genuine samples. Unlike backdooring, SMS links user-specific seeds (such as users' unique indices), original samples, and models, thereby facilitating the verification of unlearning genuine samples. However, implementing SMS for unlearning verification presents two significant challenges. First, embedding the seeds into the service model while keeping them secret from the server requires a sophisticated approach. We address this by employing a self-supervised model seeding task, which learns the entire sample, including the seeds, into the model's latent space. Second, maintaining the utility of the original service model while ensuring the seeding effect requires a delicate balance. We design a joint-training structure that optimizes both the self-supervised model seeding task and the primary service task simultaneously on the model, thereby maintaining model utility while achieving effective model seeding. The effectiveness of the proposed SMS scheme is evaluated through extensive experiments on three representative datasets, utilizing various model architectures and exact and approximate unlearning benchmarks. The results demonstrate that SMS provides effective verification for genuine sample unlearning, effectively addressing the limitations of existing solutions.

\end{abstract}

	\begin{IEEEkeywords}
		Machine unlearning, unlearning verification, self-supervised model seeding, joint optimization, backdooring.
\end{IEEEkeywords}
}

\maketitle

\IEEEdisplaynontitleabstractindextext

%
\IEEEpeerreviewmaketitle

\ifCLASSOPTIONcompsoc
\IEEEraisesectionheading{\section{Introduction}\label{sec:introduction1}}
\else
\section{Introduction}
\label{sec:introduction}
\fi


\IEEEPARstart{I}{n} recent years, numerous privacy regulations and laws, such as the General Data Protection Regulation (GDPR) and California Consumer Privacy Act (CCPA) \cite{mantelero2013eu}, have been introduced to safeguard individuals' data privacy. These legislations guarantee individuals the right to be forgotten, thus prompting a hot and attractive research topic, machine unlearning \cite{bourtoule2021machine,cao2015towards,sekhari2021remember}. Machine unlearning aims to remove the trace of user-specified samples from the already-trained models, ensuring compliance with these privacy mandates. Unfortunately, most machine unlearning techniques have overlooked providing data owners with verification to prove that their specified samples have been removed.

Enabling data owners to verify whether a ML server has unlearned their data is an important yet challenging responsibility for machine unlearning services. Existing machine unlearning approaches primarily focus on removing the influence of specified training samples from the trained model \cite{bourtoule2021machine, wu2020deltagrad, nguyen2020variational}. In these unlearning methods, the ML server usually directly executes an unlearning algorithm and releases a service interface of the model to the public. However, users often struggle to determine if their data has been unlearned, especially when they only have access to the black-box model provided by these services \cite{thudi2022necessity,ranganath2014black,zhao2021causal}. In this situation, providing an efficient unlearning verification method is essential for users to safeguard their right to be forgotten and for the ML server to build a transparent privacy-protecting environment.

\noindent
\textbf{Research Gap.}
There are a few works providing data removal verification by utilizing backdoor methods. These methods involve mixing backdoored samples with users' data for model training. After executing the unlearning algorithm, the effectiveness of backdoor attacks is used to determine whether the user's data has been unlearned~\cite{hu2022membership,sommer2022athena,guo2023verifying}. 

{However, the removal of backdoors cannot present the removal of genuine samples. It is because the backdoored samples and genuine samples are distinct datasets, exhibiting markedly different behaviors during the model unlearning processes~\cite{gao2023backdoor,zeng2023narcissus,nguyen2020input}.} The backdooring methods merely link the backdoored samples and the models. Particularly in approximate unlearning \cite{nguyen2020variational,fu2022knowledge,nguyen2022markov}, the model accuracy on backdoored samples decreases significantly faster than on genuine samples. Even when the backdoor attack success rate drops to zero, it remains challenging to ensure that the genuine samples have been unlearned. 

{Moreover, to make backdoor-based unlearning verification effective, their application is restricted exclusively to validating retraining-based (exact) unlearning methods~\cite{cao2015towards,bourtoule2021machine,yanarcane2022unlearning,wu2020deltagrad} and is suitable only for huge samples' unlearning verification as it needs to add sufficient backdoored samples~\cite{hu2022membership}.} These verification methods are not suitable for common genuine data and approximate unlearning scenarios. The difference from existing unlearning verification studies is shown in \Cref{overview_of_auditing_method}.

\begin{figure}[t]
	\centering
	\includegraphics[width=0.99\linewidth]{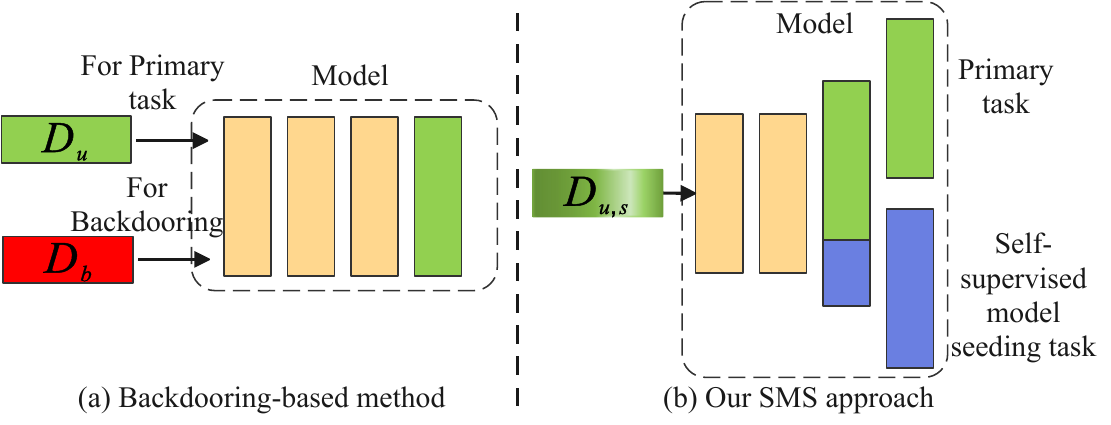}
	\caption{In backdooring-based methods, the backdoored data $D_b$ is mixed with the user's data $D_u$, serving distinct purposes: $D_u$ for the primary task and $D_b$ for backdooring. Hence, the removal of the backdoor can only verify the usage of the backdoored data, while it cannot confirm whether the users' genuine data has been unlearned. {By contrast, our method integrates seeds into the user's genuine data, resulting in $D_{u,s}$. We will not change data labels to link and highlight the seeds as backdooring. Hence, the seeds serve as normal features of $D_{u,s}$. We designed a joint training structure to learn the seeds and primary tasks simultaneously, which links the seeds, original samples, and models for unlearning verification. }
	}
	\label{fig_introfigure}
\end{figure}


\noindent
\textbf{Our Work.}
In this paper, we propose a Self-supervised Model Seeding scheme (SMS) designed for verification of machine unlearning, suitable for genuine samples. We present an overview of the SMS structure and a comparison with the backdoor-based methods~\cite{hu2022membership,sommer2022athena} in \Cref{fig_introfigure}. SMS effectively verifies the unlearning of genuine samples as it establishes the link between seeds, original samples, and models. Our SMS approach just embeds seeds into users' data without altering the labels, thus cutting off the reliance on additional samples, which is different from backdoor-based methods. However, implementing SMS for unlearning verification still faces two significant challenges. First, to enable the model effectively learns the seed information from the seed-embedded data, we design a self-supervised model seeding task that embeds the entire data (including the seed information) into the learned latent space. Second, to mitigate the negative impact of the model seeding task on the primary task, we propose a joint optimization structure that captures seed-related information while preserving the models' utility on the primary task. After the model seeding training is finished, the users can train a verifier to identify if the model outputs contain the seed information, enabling users to verify unlearning operations.

We conducted extensive experiments to evaluate the SMS scheme for unlearning verification across three representative datasets, various model architectures, and mainstream exact and approximate unlearning benchmarks. The results demonstrate the superiority of our approach in supporting genuine data verification at both the sample and user levels, in both exact and approximate unlearning scenarios, compared to the state-of-the-art methods. Furthermore, SMS effectively mitigates model utility degradation, outperforming current backdoor-based methods.


Our main contribution is summarized as follows:
\begin{itemize}[itemsep=0pt, parsep=0pt, leftmargin=*]
	\item We recognize the critical need for unlearning verification and formalize the machine unlearning verification problem. This formalization elucidates the requirements for data removal verification and informs the design of effective unlearning verification strategies.
	\item We propose a SMS scheme to empower data owners to verify whether the server unlearns their data. In this approach, the seeds are integrated into the data as inherent features. The server executes the SMS training using the designed joint-optimization structure without requiring any additional samples. 
	\item We conducted extensive experiments on both exact and approximate unlearning benchmarks across three representative datasets and various model architectures. The results demonstrate the superiority of SMS over existing solutions, supporting a broader range of unlearning scenarios, especially approximate unlearning scenarios, where the backdoor-based methods are infeasible.   
	\item The source code and artifacts of the SMS scheme are released at \url{https://github.com/wwq5-code/SMS}, which provides an effective model seeding tool for machine unlearning verification, achieving significant improvements over existing verification methods and shedding light on the design of future unlearning verification strategies. 
\end{itemize}

%

\noindent  \textbf{Roadmap.} 
The rest of the paper is structured as follows: Section \ref{rw} reviews related work and some background knowledge. We introduce a detailed problem statement of data usage verification for machine unlearning in Section~\ref{ps}. The detailed implementation of our proposed SMS is presented in Section \ref{vmu_method}. Section \ref{ex} showcases our experimental results and comparison with related works. Finally, in Section \ref{summary}, we summarize the paper and discuss the future work.

\begin{table}[t]
	\scriptsize
	\caption{An overview of unlearning verification methods.}
	\label{overview_of_auditing_method}
	\resizebox{\linewidth}{!}{
		\setlength\tabcolsep{3.pt}
		\begin{tabular}{c|cccc}
			\toprule[1pt]
			\multirow{2}{*} { \makecell[c]{\textbf{Unlearning} \\ \textbf{Verification} \\ \textbf{Methods}} } & \multicolumn{2}{c} { \textbf{Verifying Data Type}} & \multicolumn{2}{c} {\textbf{Unlearning Methods}} \\
			\cmidrule(r){2-3}   \cmidrule(r){4-5} 
			& \makecell[c]{{Backdoored} \\ {samples	}  }    & \makecell[c]{{Genuine} \\ {samples	}  }   & \makecell[c]{{Exact} \\ {unlearning}  }    &\makecell[c]{{Approximate} \\ {unlearning}  }  \\ 
			\midrule
			MIB~\cite{hu2022membership} & \filledcircle&\emptycircle	&\filledcircle & \emptycircle   \\
			Athena~\cite{sommer2022athena} & \filledcircle  &\emptycircle 	&\filledcircle &\emptycircle    \\
			Verifying in the dark~\cite{guo2023verifying} & \filledcircle  &\emptycircle  &\filledcircle & \emptycircle  \\
			Verifi~\cite{gao2024verifi} & \filledcircle  &\emptycircle  	&\filledcircle & \emptycircle   \\
			SMS (Ours)	     & \emptycircle  & \filledcircle  	&\filledcircle  &\filledcircle   \\
			\bottomrule[1pt]
	\end{tabular}}
	\begin{tabbing}
		\filledcircle: the verifying method is applicable; 
		\emptycircle: the verifying method is not applicable.
	\end{tabbing}
	\vspace{-2mm}
\end{table}

\section{Related Work} \label{rw}

\subsection{Machine Unlearning and Verification of Unlearning}


Unlearning methods can be broadly categorized into two types: exact unlearning and approximate unlearning. Exact unlearning methods extend from naive retraining, aiming to reduce the computational burden of training a new model by redesigning learning algorithms and storing intermediate parameters during the learning process~\cite{cao2015towards,bourtoule2021machine,yanarcane2022unlearning,wu2020deltagrad}. Although exact unlearning methods achieve effective unlearning, they still incur substantial computational costs, rendering them impractical in scenarios where unlearning requests are frequent \cite{Hu2024sp,hu2024eraser}. Conversely, approximate unlearning methods seek to adjust the model parameters based on the original model, with the objective of approximating a model as if it had been retrained on the remaining dataset~\cite{nguyen2020variational,fu2022knowledge,nguyen2022markov,zhang2024forgetting}. While approximate unlearning methods can efficiently implement unlearning, they compromise unlearning effectiveness and model utility to some extent.

Since unlearning requests frequently involve user privacy concerns, it is crucial to provide verification mechanisms to users, demonstrating that their data has been successfully removed from the trained model~\cite{thudi2022necessity,sommer2020towards}. {Membership inference attack (MIA) is commonly used as an evaluation metric for the model server to evaluate the unlearning algorithms \citep{zhao2024makes,kurmanji2023towards}. However, MIA usually can only provide a rough prediction as the training set and other test sets usually have similar samples or features. Moreover, conducting MIA from the user's side is complex as it requires users having the adversary's capability~\citep{wang2025crfu,chen2021machine} to mimic the training model and prepare the probe data.}


Backdoor-based verification methods~\cite{gao2024verifi,hu2022membership,guo2023verifying,sommer2022athena} are specifically suited for retraining-based unlearning approaches, also known as exact unlearning~\cite{cao2015towards,bourtoule2021machine,yanarcane2022unlearning}. These methods exploit the intrinsic property of exact unlearning, which involves the concurrent deletion of both backdoored and erased samples at the dataset level, thereby ensuring the model is retrained exclusively on the residual dataset~\cite{cao2015towards,bourtoule2021machine}. 

However, for approximate unlearning methods, which lack the guarantee that the model is retrained based on the remaining dataset excluding the backdoored and erased samples \cite{zhang2024forgetting,warnecke2024machine}, these backdoor-based verification methods can only ascertain the inclusion or exclusion of backdoored samples rather than genuine samples. The distinction between the backdoored dataset and the genuine dataset in approximate unlearning scenarios will be further elaborated in \Cref{unlearning_opt_process} in \Cref{verify_of_diff_unl}. Thus, there is an imperative need to develop a practical and verifiable scheme for machine unlearning that accommodates these typical unlearning requests.


\subsection{Difference form Existing Studies}

Our SMS approach diverges significantly from existing backdoor-based unlearning verification methods~\cite{gao2024verifi,hu2022membership,guo2023verifying,sommer2022athena} concerning the type of verifying data and the applicable unlearning methods, as depicted in \Cref{overview_of_auditing_method}. 

Firstly, backdoor-based verification methods necessitate the integration of backdoored samples into users' datasets. However, backdoored and genuine samples constitute distinct datasets, and the elimination of the backdoor merely signifies the unlearning of backdoored samples, not genuine ones. Conversely, our method embeds seeds into genuine samples without changing the labels. Our method builds links between seeds, genuine samples and models using designed model seeding methods, thus ensuring the verification of genuine samples. 

Secondly, backdoor-based methods are restricted to exact unlearning methods due to their dependence on dataset-level data removal characteristics. Backdooring techniques can only establish the connection between backdoored samples and models. These backdoor-based verification methods can only be employed for unlearning verification when exact unlearning ensures the simultaneous deletion of both backdoored and erased genuine samples at the dataset level. Our method, however, directly integrates seeds into genuine samples without requiring any additional samples and is applicable to both exact and approximate unlearning methods, providing a more versatile and comprehensive verification solution.

\section{Scenario and Problem Statement} \label{ps}



\subsection{The MLaaS Scenario}
To facilitate understanding, we introduce our problem within a simple Machine Learning as a Service (MLaaS) scenario. In this scenario, there are two primary entities: an ML server that collects data from users and trains ML service models, and users (data owners) who contribute their data for ML model training.

To foster a healthy privacy-protection environment and support the execution of the right to be forgotten, the ML service server is responsible for executing unlearning. However, it is challenging to achieve sufficient evidence for users to verify if users' data are unlearned to prevent the spoof unlearning from the ML server. Similar to the most common unlearning verification settings \cite{hu2022membership,guo2023verifying}, we assume the ML server is honest for learning training but may spoof users for unlearning, i.e., while it honestly hosts learning processes, it may cheat during unlearning operations by pretending that unlearning has been executed when it has not. {It is reasonable for the ML server to pretend to execute unlearning operations because pretending to unlearn while not executing can avoid the degradation of model utility.}

{Specifically, for the ML server in this paper, it knows the total training dataset $D$, the service model $M$, and learning and unlearning algorithms. The restriction is that the server does not know the seed information of users, necessitating the design of general model seeding techniques for unlearning verification without revealing the seed to the server. For the unlearning users, they only have access to their unlearning data, the seed information, and the black-box access to the model, i.e., querying the model with inputs and getting corresponding outputs. The unlearning users cannot access the model parameters and the remaining dataset.}



\subsection{Unlearning Verification Problem}
Based on the above scenario, we first introduce the main process of unlearning to facilitate understanding of the unlearning verification problem.

\vspace{1mm}
\noindent
\textbf{Machine Unlearning.} The unlearning process usually includes the following steps. (1) The ML server trained a model $M$ based on the collected dataset $D$. (2) The user uploads the unlearning requested erasure dataset $D_e$ to the server for unlearning. (3) The server conducts an unlearning algorithm $\mathcal{U}$ to remove $D_e$'s contribution from $M$.


The unlearning verification problem aims to confirm whether previously used data in step (1) has been removed from the model following the execution of unlearning algorithms after step (3). Specifically, a data owner's samples are initially employed to train the model, and upon submitting an unlearning request and completing the unlearning operation, it should be verified that the specified samples have been removed from the model.


\begin{prob_state}[Verification of Machine Unlearning]
	Assume there is a model $M$ trained with users' data. When the user requests to unlearn the specified data $D_{e}$, the previously used data $D_e$ should disappear in the model after executing the unlearning algorithm $\mathcal{U}$. It requires us to have a verifying method $\mathcal{V}$, which can identify the data usage before unlearning, $\mathcal{V}(M, D_e)=1$, and verify data removal after unlearning, $\mathcal{V}(\mathcal{U}(M), D_e)=0$. Meanwhile, the verification scheme should not compromise the model utility of the original ML service.
\end{prob_state}
If $D_e$ only contains a few samples of the user, we can also call it a \textit{sample-level} unlearning verification problem. Directly solving this problem is not easy. Hence, current solutions are based on backdooring techniques \cite{guo2023verifying, hu2022membership}. However, these methods are infeasible for verifying unlearning in approximate unlearning scenarios and for genuine samples. We propose the model seeding strategy to solve the unlearning verification problem in this paper.


\subsection{Defining the Requirements of Model Seeding Scheme for Unlearning Verification} \label{defining_requirements}

We now define the requirements for an effective unlearning verification solution, focusing on verifiability, unambiguity, and functionality preservation. 

\subsubsection{Verifiability}

To verify whether data has been unlearned from models, the model seeding scheme must include two embedding steps: integrating seeds into the data to mark it and embedding the marked data into the model for subsequent unlearning verification. The process can be summarized as follows. First, the data owner generates a seed for the sample $x_i$ using a generation module $\texttt{Gen}$ with security parameter $N$:
\[
s_i \gets \texttt{Gen}(1^N).
\]
Then, the data owner embeds the seed $s_i$ into the sample $x_i$. The server subsequently embeds the seed into the model using the seed-embedded sample $x_{i,s_i}$. The data seeding and model seeding can be described as follows.
\begin{align*}
	& \textbf{Data Seeding: }  \hspace{17mm}	x_{i,s_i}   \gets \texttt{E}_1(x_i, s_i),  \\
	& \textbf{Model Seeding: } \hspace{8mm} (M_{\texttt{S}}, x_{i,s_i}, \mathcal{V})  \gets \texttt{E}_{2}(M, \texttt{E}_{1}(x_i, s_i)), 
\end{align*}
where $\texttt{E}_1$ is the data seeding module that integrates the seed into the data, and $\texttt{E}_2$ is the model seeding module that embeds the seed of seed-embedded data into the model. Here, $x_{i,s_i}$ represents the seed-embedded data, $M_{\texttt{S}}$ denotes the seed-embedded model, and $\mathcal{V}$ is the verification model used to identify the seeds.

An effective unlearning verification scheme must have verifiability, ensuring that when the verifying function $\mathcal{V}$ checks the model $M_{\texttt{S}}$ with the correct seed input $x_{i,s_i}$, the probability of correctly verifying the presence of the seed is high. Formally, this can be defined as: 
\begin{defn}[Verifiability] 
Verifiability is achieved if, for any given seed $s_i$ generated with the security parameter $N$ by $\texttt{Gen}$, the probability that the verifying function $\mathcal{V}$ correctly asserts the presence of the seed in $M_{\texttt{S}}$ using $x_{i,s_i}$ is at least $1 - \epsilon^{-N}$. Mathematically, this is expressed as: 
	{ 
		\begin{align}
			&\Pr \{\mathcal{V}(M_{\texttt{S}}, x_{i,s_i})=1\} \geq 1 - \epsilon^{-N}. \label{veri} 
		\end{align}
	}
\end{defn}
In \Cref{veri}, $\epsilon > 1$ and $\epsilon^{-N}$ declines exponentially as $N$ increases. A higher $N$ (the security parameter) results in a smaller $\epsilon^{-N}$, thereby enhancing verifiability. However, as the security parameter increases, the seed becomes more noticeable. For example, an MNIST data sample contains $28 \times 28$ pixels. If the seeding strategy involves setting $N$ pixels to 1, a larger $N$ will result in more pixels being set to 1, making the seed more apparent. Consequently, verifiability is higher because a larger and more visible seed is easier to identify. The trade-off is that the seed's ability to remain concealed is diminished.

Assuming the scheme possesses an excellent verification mechanism to check data usage in $M_{\texttt{S}}$, the effectiveness of unlearning can be represented by the disappearance of the seed check after unlearning. This is characterized by a high $\Pr \{\mathcal{V}( \mathcal{U} (M_{\texttt{S}}), x_{i,s_i})=0\}$ or a low $\Pr \{\mathcal{V}( \mathcal{U} (M_{\texttt{S}}) , x_{i,s_i})=1\} $.


\subsubsection{Unambiguity}

Unambiguity ensures that when the verifying function $\mathcal{V}$ checks the model $M_{\texttt{S}}$ with an incorrect seed input $x_{i,s_i'}$, the probability of mistakenly verifying the presence of the seed $s_i$ should be very low. Formally, this can be defined as: 
\begin{defn}[Unambiguity]
Unambiguity is achieved if, for any non-original seed $s_i'$ (different from $s_i$), the probability that the verifying function $\mathcal{V}$ for $s_i$ correctly denies the presence of $s_i'$ in $M_{\texttt{S}}$ using $x_{i,s_i'}$ is at least $1 - \epsilon^{-N}$, where $\epsilon^{-N}$ is a small positive function that decreases exponentially as $N$ increases. Mathematically, this is expressed as: 
	{ \small
		\begin{align}
			\Pr \{\mathcal{V}(M_{\texttt{S}}, x_{i,s'_i})=0\} \geq 1 -  \epsilon^{-N}. \label{unambi}
		\end{align}
	}
\end{defn}
This definition implies that the model seeding scheme is highly reliable in both confirming the presence of the correct seed and avoiding the misidentification of incorrect ones. Consequently, it can effectively defend against the ambiguity attacks like ~\cite{fan2019rethinking,ong2021protecting}.

To claim the usage of data $x_{i,s_i}$ for model $M_{\texttt{S}}$ training with a verifier $\mathcal{V}$, both verifiability (\Cref{veri}) and unambiguity (\Cref{unambi}) are required. Verifiability ensures that $\mathcal{V}$ can accurately decode and verify the embedded seed $s_i$. Unambiguity ensures that the seeding scheme is resistant to ambiguity attacks~\cite{hua2023unambiguous}. An ambiguity attack involves using a forged seed $s_i'$ to cheat the $\mathcal{V}$ to achieve the ownership verification \cite{fan2019rethinking,ong2021protecting,hua2023unambiguous}. Unambiguity requires that $\mathcal{V}$ does not misidentify other seeds as the genuine seed. 


 \begin{figure*}[t]
 	\centering
 	\includegraphics[width=0.96\linewidth]{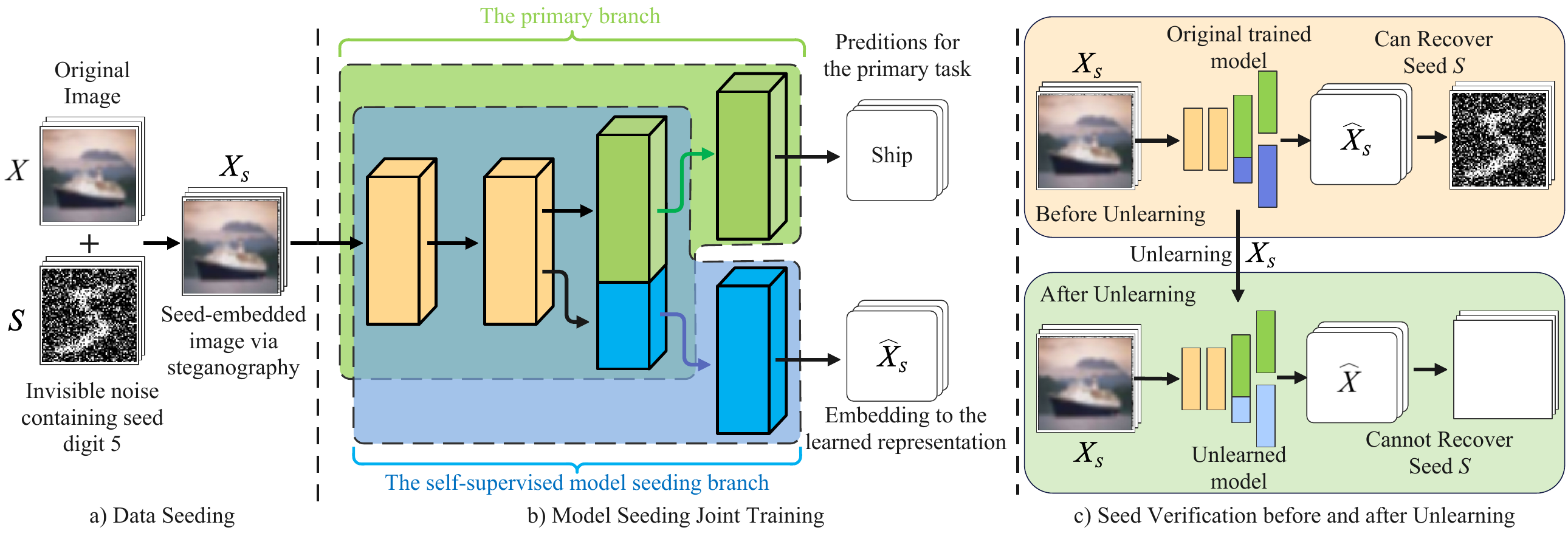}
 	\caption{Overview of the SMS Scheme. The SMS scheme consists of three main phases: (a) Users generate seeds and integrate them into their data. (b) The ML server jointly trains the model for both the primary task (green block) and the self-supervised model seeding task (blue block), with the overlapped area representing the shared network layers of the model. (c) Users verify if their seeds have disappeared from the model following unlearning requests.}
 	\label{fig_vmuviawatermark}
 \end{figure*}

\subsubsection{Functionality Preservation}
The model seeding scheme should not come at the cost of significant model utility degradation. In other words, the performance of the seed-embedded model should be only slightly worse than, or equivalent to, that of the originally trained model. The formal definition is 
\begin{equation} \label{functional_preserve}
	\Pr\nolimits_{(x,y) \sim \mathcal{T}_{\text{primary}}} \{d(M(x), M_{\texttt{S}}(x)) \leq \delta \}  \approx 1,
\end{equation}
which means that the seed-embedded model must perform similarly to the original model from the primary task perspective. For ease of understanding, some studies smooth the \Cref{functional_preserve} as 
\[
\forall x \in \mathcal{X}, d(M(x), M_{\texttt{S}}(x)) \leq \delta.
\]
This equation is easier to solve than \Cref{functional_preserve}, as it only requires checking that the output distribution of $M_{\texttt{S}}$ does not deviate significantly from that of the originally trained model $M$. In this paper, we additionally employ a self-supervised model seeding task, which differs from the primary task of the original model service. Therefore, we prefer using \Cref{functional_preserve} as the formalization of functionality-preservation. This means that while the parameters of the seed-embedded model $M_{\texttt{S}}$ may differ from those of the originally trained model, it must perform similarly in the primary task. Additionally, the method should satisfy the previously defined verifiability and unambiguity requirements.

\section{Self-Supervised Model Seeding for Verification of Machine Unlearning} \label{vmu_method}

\subsection{Overview of SMS}
Here, we provide a brief overview of the primary concept and process of our approach. The main procedure of the proposed Self-supervised Model Seeding (SMS) scheme is illustrated in \Cref{fig_vmuviawatermark}. The SMS scheme comprises three main phases: data seeding, model seeding joint training, and seed verification. Firstly, the data owner embeds their data with unique seeds and uploads the seed-embedded data to the server for ML model training. Secondly, the model seeding joint training phase is designed to simultaneously train both the primary task and the self-supervised model seeding task, ensuring that the model learns the seed while preserving the utility of the primary task. Finally, the seed verification phase allows the data owner to verify whether their data has been used in the model training and unlearned after unlearning.

\subsection{Data Seeding}

Our final goal of SMS is to establish the links between seeds, genuine samples, and models so that we can verify the unlearning of genuine samples from the seed status of the model. Therefore, we first embed the seeds into the data, making the seed information an integral feature of the data for model training. The generated seeds can be invisible additive noises containing secrecy. The secrecy can be flexibly designed by the user, such as a name, a user index, or a generated key with a random seed. In this paper, for convenience, we use a small digit image with noise as seeds $s$, where $s \in \mathcal{S}$. The seed $s$ belongs to the seed space $\mathcal{S}$, which is a subspace of the input data space $\mathcal{X}$.

Assume we have a data seeding function $\texttt{E}_1(\cdot, \cdot)$. To implement the data seeding function $\texttt{E}_1(\cdot,\cdot)$, we can directly embed the seed into the data using image steganography techniques~\cite{baluja2019hiding,tamimi2013hiding}. Alternatively, we can use a simpler method similar to adding patches, as employed in backdoor methods~\cite{chen2017targeted,wang2019neural,lin2020composite}. The function $\texttt{E}_1(\cdot, \cdot)$ can thus be defined as: 
\begin{align} \label{eq_embed1}
	\texttt{E}_1(x, s) = (1-v) \otimes x + v \otimes s,
\end{align}
where $\otimes$ represents the element-wise product, and $v$ is a mapping parameter with the same form as $x$, where each element ranges within $[0,1]$. By adjusting the value of $v$ in \Cref{eq_embed1}, we can set the seeds to an invisible level. The tradeoffs involved in this adjustment will be discussed in more detail in the experimental section.

\subsection{Model Seeding Joint Training}
Embedding seeds into the model's latent space using seed-embedded data is considerably more challenging than backdooring the model. The backdoor method involves mixing backdoored samples, which contain a designed trigger and altered labels associated with that trigger. Backdoor-based methods do not require modifications to the training process, but the altered labels of samples for training compromise model utility. In contrast, model seeding methods only introduce seeds into the data without altering the labels, thereby preserving data usability. However, this approach presents a greater challenge in embedding the seed-embedded data into the model compared to direct backdoor training.

\vspace{1mm}
\noindent
\textbf{Primary Task.} We have introduced a data seeding method that integrates seeds as features into users' data. To embed seeds into the model while mitigating utility degradation, the proposed method does not alter the labels used for the ML service task, referred to as the primary task (green block in \Cref{fig_vmuviawatermark}). The objective of this approach is described as: 
\begin{equation} \label{primary_loss}
	\theta^* = \arg\min_{\theta}  \frac{1}{N}\sum_{i=1}^N\mathcal{L}_{primary}( M(x_{i,s_i};\theta), y_i),
\end{equation}
where $\theta$ represents the parameters of the model $M$, $x_{i,s_i}$ is the seed-embedded sample generated using the previously introduced techniques, and $\mathcal{L}_{primary}(M(x_{i,s_i};\theta), y_i)$ denotes the primary loss incurred when predicting $M(x_{i,s_i};\theta)$ with the true label $y_i$. This primary loss function is typically implemented using cross-entropy loss~\cite{goodfellow2016deep,zhang2018generalized}. 

\vspace{1mm}
\noindent
\textbf{Self-supervised Model Seeding.} To embed the seed in the latent space without the assistance of the labels, we designed a novel approach by incorporating a self-supervised model seeding task. This task aims to learn detailed features of the data, including seeds. The objective of this task is described as follows:   
\begin{equation} \label{self_loss}
	\theta^* = \arg\min_{\theta}  \frac{1}{N}\sum_{i=1}^N\mathcal{L}_{self}( M(x_{i,s_i};\theta), x_{i,s_i}),
\end{equation}
where $\mathcal{L}_{self}(M(x_{i,s_i};\theta), x_{i,s_i})$ represents the self-supervised model seeding loss function that learns the features of the seed-embedded samples. This self-supervised model seeding loss function can be implemented using the loss mechanism of an autoencoder~\cite{baldi2012autoencoders,mescheder2017adversarial}. One advantage of using self-supervised learning to embed seeds, rather than a classification task with seed label, is that it does not require disclosing additional information about the seed to the server. All seed information is integrated into the data as features. Thus, by designing server-invisible features as seeds, we can effectively conceal the seeds from the server. 



\vspace{1mm}
\noindent
\textbf{Seeding-Joint Training.} As shown in \Cref{fig_vmuviawatermark}, we design the primary and self-supervised model seeding tasks to share most parameters, except for the model's last layer. We optimize these two tasks using a two-objective learning method inspired by \cite{sener2018multi,lin2019pareto}. One simple approach is to introduce two weights, $\alpha^p$ and $\alpha^s$, to optimize them as follows:
\begin{equation} \label{multi_opt}
	\arg \min_{\theta} \  \alpha^p \mathcal{L}_{primary} + \alpha^s \mathcal{L}_{self},
\end{equation}
where the optimized $\alpha^p$ and $\alpha^s$ can be achieved using muti-objective optimization methods as did in \cite{sener2018multi}. In this optimizing way, we can ensure the seeds embedding effect while mitigating the model utility degradation for the primary task.

\begin{algorithm}[t]
	\caption{Self-supervised Model Seeding (SMS)} \label{VMU}
	\begin{small} 
		\BlankLine
		\KwIn{$D_{train}$, and epochs $E$}
		\KwOut{The seed-embedded model $M_{\texttt{S}}$ and seed verifying model $\mathcal{V}$}
		\textbf{Data Seeding:} \hspace{2.2cm} $\rhd$ Run on the user $u$ side \\
		\For{sample $x_i$ in $D_u$}{
			Generate seed $s_i$ for sample $x_i$; \\ 
			Embed the seed $s_i$ into the sample $x_i$ using \Cref{eq_embed1}: $x_{i,s_i} = (1-v) \otimes x_i + v \otimes s_i$; \\
		}
		Upload the seed-embedded dataset $D_{u,s}$ to server for model training; \\
		\textbf{Model Seeding Joint Training:} \hspace{1cm} $\rhd$ Run on the server side \\
		Collect all datasets from users and prepare the training dataset $D_{train} = \{D_{1,s_1}, ..., D_{u,s_u}, .., D_{n, s_n}\}$; \\
		Initialize the model $M$ with parameters $\theta$; \\
		\For{$E$ epochs}{
			Sample minibatch of $m$ data points $\{(x_{i,s_i}, y_i)\}_{i=1}^m$ from the collected whole dataset $D_{train}$; \\
			Calculate the loss primary loss and self-supervised training loss of the sampled minibatch using \Cref{primary_loss,self_loss};   \\
			Update model parameters using \Cref{multi_opt}:
			$\newline \theta \gets \theta - \eta \nabla_{\theta} (\alpha^p \mathcal{L}_{primary} + \alpha^s \mathcal{L}_{self})$   \\
		}
		Achieve the seed-embedded model $M_{\texttt{S}}$ \\
		\textbf{Seed Verification:} \hspace{14mm}  $\rhd$ Run on the user $u$ side\\
		Initialize and prepare the auxiliary $D_{veri.}$ \\
		\For{$x_i$ in $D_u$ and $x_{i,s_i}$ in $D_{u,s}$}{
			$D_{veri.}$ adds the negative sample $(x_i, 0)$ \\
			$D_{veri.}$ adds the positive sample  $(x_{i,s_i},1)$ \\
		}
		Initialize the verification model $\mathcal{V}$ and train it using $D_{veri.}$ \\
		\textbf{Verification before Unlearning:} \\
		Achieve the output of the ML model using $x_{i,s_i}$: 
		$( \hat{y_i}, \hat{x_{i,s_i}}) \gets M_{\texttt{S}}(x_{i,s_i})$ \\
		Verify result $r \gets \mathcal{V}(\hat{x_{i,s_i}})$  \\
		\textbf{Unlearning Verification:} \\
		The user upload the unlearning requirements of $x_{i,s_i}$, and the server unlearns $x_{i,s_i}$ from the ML model using an unlearning algorithm $\mathcal{U}$: $M_{\texttt{S}}^{\mathcal{U}}  \gets  \mathcal{U}(M_{\texttt{S}}, x_{i,s_i})$ \\
		Achieve the output of the unlearned model using $x_{i,s_i}$: 
		$( \hat{y_i^u}, \hat{x_{i,s_i}^u}) \gets M_{\texttt{S}}^{\mathcal{U}}(x_{i,s_i})$ \\
		Verify result $r \gets \mathcal{V}(\hat{x_{i,w_i}^u})$  \\
		Return $M_{\texttt{S}}$ and $\mathcal{V}$;
	\end{small}
\end{algorithm}

\subsection{Seed Verification}
After embedding the unique seed into the model latent space through the seeding-joint training phase, the server can provide evidence of data usage by demonstrating whether the seed information exists in the model. The most straightforward method is to verify if the self-supervised model's output contains the seed when the input is the seed-embedded sample $x_{i,s_i}$. If the output of the self-supervised model includes the seed $s_i$, it indicates that the model retains the contribution of the specified sample $x_{i,s_i}$. Conversely, if the seed is absent, it signifies that the model has unlearned the specified sample.

The output of the self-supervised model seeding task is the reconstructed sample based on the input. We train the self-supervised model as an auto-encoder, which learns all the information about the samples, including the hidden seeds. Based on the self-supervised output, a verifying model is individually trained by the user to identify the designed seeds for data removal verification.

Since the data owner possesses the original sample and knows the seeds, they can directly label the original sample $x$ as negative and the seed-embedded sample $x_{s}$ as positive. The auxiliary dataset is denoted as $D_{veri.}$. The data owner can then train a classification model $\mathcal{V}$ based on $D_{veri.}$. When the user needs to identify if the model has used their samples, they use this model to calculate the probability $\Pr \{\mathcal{V}(M_{\texttt{S}}, x_{i,s})=1\}$ for data usage verification. If the verification probability is high and then becomes low after unlearning, it means the data is first used in the model training and then removed from the model by unlearning. Typically, the classification model $\mathcal{V}$ is easy to implement, such as using simple linear neural networks. The pseudocode for the entire SMS process is presented in \Cref{VMU}.

Specifically, the data seeding is conducted on the user side. The user samples data point $x_i$ from their local data $D_u$, generates a seed $s_i$ for the sampled data, and embeds the seed $s_i$ into sample $x_i$ using \Cref{eq_embed1}. The user then uploads the seed-embedded dataset to the server for model training. The server performs model seeding-joint training to learn both the primary task knowledge and the seed information from the users' data, preparing for data usage and removal verification. The server collects all datasets from users and initializes the model $M$ with parameters $\theta$. The model is trained through a standard process but incorporates two loss functions: the primary loss and the self-supervised model seeding loss. The seed verification is executed on the user side. Users train their unique Verifier $\mathcal{V}$ to identify their own seeds and utilize the trained $\mathcal{V}$ to assess the unlearning effect.

{ 
\subsection{Theoretical Analysis}

In this section, we provide a theoretical analysis from the mutual information perspective for the SMS verification, i.e., the disappearance of seeds is equivalent to complete data forgetting, which at least provide much more information than backdoor label changing for verification.

\begin{proposition}[Strictness of SMS Verification]
Let $x \in \mathcal{X}$ be a genuine user sample with label $y \in \mathcal{Y}$ and $s \in \mathcal{S}$ be a user-generated seed and  $x_s = \texttt{E}_1(x, s)$ be the seed-embedded sample using \Cref{eq_embed1}. Let $z$ be the latent representation that is jointly optimised for the primary task and the self-supervised reconstruction task.
Since $s$ is not connected or highlighted with a label in primary task, but treated as normal $x$ feature by self-supervised task, then, the erasure mutual information of $I(z; x_s)$ provide a much more strictness than the erasure of $I(z; y)$; formally, $I(z; y) \ll I(z; x_s)$.
\end{proposition}

\begin{proof}
By the data-processing inequality, we can get the uppoer bound on label information $I(z;y) \leq H(y) \ll I(z;x)$. For an auto-encoder with distortion $D = \mathbb{E} \lVert x_s - \hat{x}_s\rVert^{2}$, one obtains the classic VAE~\cite{kingma2014auto} bound $I(z; x_s) \ge H(x_s) - \beta D$ ($\beta$-VAE/ELBO). Since $H(x_s) \approx H(x) + H(s \mid x)$,  we have $I(z; y) \ll I(z; x_s)$.
\end{proof}

Since the seed via the self-supervised objective encodes $I(z;x_s)$ (virtually the full input entropy) into the same representation shared by the primary task, eliminating the seed forces the model to uproot those high‑information features. By contrast, a backdoor unlearning step only needs to break the narrow trigger to wrong-label shortcut. Hence, ``seed disappearance'' constitutes a stricter and safer forgetting criterion: once the seed can no longer be verified, any label‑level information about that sample has already been removed.
}

\section{Performance Evaluation} \label{ex}



\begin{table*}[t]
	\scriptsize
	\caption{Overall Evaluation Results of Verification for Unlearning Genuine Samples on MNIST, CIFAR10 and CelebA.}
	\label{tab_total}
	\resizebox{\linewidth}{!}{
		\begin{tabular}{c|ccccccccc}
			\toprule[1pt]
			\multirow{2}{*} {Learning Verification} & \multicolumn{3}{c} {MNIST, ${\it SSR}=0.6\%$}& \multicolumn{3}{c} {CIFAR10, ${\it SSR}=0.6\%$} & \multicolumn{3}{c} {CelebA, ${\it SSR}=0.6\%$}  \\
			\cmidrule(r){2-4}   \cmidrule(r){5-7} \cmidrule(r){8-10}
			& Non-Verif.	& MIB   & SMS		 	& Non-Verif. 		& MIB   &SMS  		 & Non-Verif.   	& MIB   &SMS  \\
			\midrule 
			Model Acc.	& 99.39\%  & 99.31\%  &\textbf{99.46\%}         & 89.25\%	  & 89.24\%   		&\textbf{89.57\%} & 97.22\% & 97.02\%	&\textbf{97.28}\% \\
			Verifiability   & -   			&100\%   	&\textbf{100\%}   				& - 			 & \textbf{100\%} &97.67\%				& -  	 	& 93.27\% & \textbf{95.38\%} \\
			Unambiguity  & -		 & 100\%         & \textbf{100\%}   		  &-	          & 98.67\% 			& 96.67\%				& - 		&57.38\%  		& \textbf{90.77\%} \\
			Running time (s)  & \textbf{2087} & 2593 & 2873   & \textbf{2289}         &2627  & 3263 		& \textbf{3015} 		&3421 		& 4225 \\
			\midrule[0.12em]
			\multirow{2}{*} { \makecell[c]{Retraining Unlearning \\Verification} } & \multicolumn{3}{c} {MNIST, ${\it SSR}=0.6\%$}& \multicolumn{3}{c} {CIFAR10, ${\it SSR}=0.6\%$} & \multicolumn{3}{c} {CelebA, ${\it SSR}=0.6\%$}  \\
			\cmidrule(r){2-4}   \cmidrule(r){5-7} \cmidrule(r){8-10}
			& Non-Verif.	& MIB   &SMS		 	& Non-Verif. 		& MIB   &SMS  		 & Non-Verif.   	& MIB   &SMS  \\
			\midrule 
			{Model Acc.}	& 99.32\%  & 99.33\%  &\textbf{99.43\%}         & 89.22\%	  & 89.39\%   	& \textbf{89.52\%} & 97.18\% & 96.83\%	& \textbf{97.27\%} \\
			{Verifiability}   & -   & \textcolor{red}{100\%}   &\textbf{0\%}   & -  & \textcolor{red}{100\%} &\textbf{0\%}		& -  	 		& \textcolor{red}{93.12\%} & \textbf{2.05\%} \\
			{Unambiguity}  & -  & 100\%         & \textbf{100\%}   		  &  -	          & 99.99\%  & \textbf{100\%} 						& - 		 &54.29\%  		& \textbf{99.12\%} \\
			{Running time (s)}  & \textbf{2070} & 2543 & 2839   & \textbf{2269}         &2612  & 3238 		& \textbf{2994} 		&3401 		& 4198 \\
			\midrule[0.12em]
			\multirow{2}{*} {\makecell[c]{SISA Unlearning \\Verification }} & \multicolumn{3}{c} {MNIST, ${\it SSR}=0.6\%$}& \multicolumn{3}{c} {CIFAR10, ${\it SSR}=0.6\%$} & \multicolumn{3}{c} {CelebA, ${\it SSR}=0.6\%$}  \\
			\cmidrule(r){2-4}   \cmidrule(r){5-7} \cmidrule(r){8-10}
			& Non-Verif.	& MIB   &SMS		 	& Non-Verif. 		& MIB   &SMS  		 & Non-Verif.   	& MIB   &SMS  \\
			\midrule 
			Model Acc.	& 99.29\%  & 99.33\%  &\textbf{99.39\%}         & 89.12\%	  & 89.37\%   	& \textbf{89.44\%} & 97.15\% & 96.77\%	& \textbf{97.26\%} \\
			Verifiability   & -   & \textcolor{red}{100\%}   &\textbf{0\%}   & -  & \textcolor{red}{100\%} &\textbf{0\%}		& -  	 		& \textcolor{red}{92.89\%} & \textbf{3.41\%} \\
			Unambiguity  & -  & 100\%         & \textbf{100\%}   		  &  -	          & 99.32\%  & \textbf{100\%} 						& - 		 &53.72\%  		& \textbf{98.21\%} \\
			Running time (s)  & \textbf{410} & 503 & 562   & \textbf{451}         &520  & 640 		& \textbf{589} 		&677 		& 823 \\
			\bottomrule[1pt]
	\end{tabular}}
\end{table*}

\subsection{Experimental Settings}


\noindent
\textbf{Datasets.} 
We conducted experiments on three widely-used public datasets: MNIST, CIFAR10, and CelebA. For MNIST and CIFAR10, the goal was to train a model for 10-class classification. For CelebA, we focus on binary classification to determine gender attributes in face images.


\vspace{1mm}
\noindent
\textbf{Models.} 
We use two model architectures of different sizes in our experiments: a 5-layer multi-layer perceptron (MLP) connected by ReLU and ResNet-18. Specifically, the verification model $\mathcal{V}$ for all datasets is trained using the 5-layer MLP model. The input to $\mathcal{V}$ is the self-reconstructed samples produced by the self-supervised model seeding task. To implement the model seeding joint training structure, we employ ResNet-18 as shared layers to learn a robust data representation. Following the ResNet-18 model, we connect two 5-layer MLP models: one for the primary task and one for the self-supervised model seeding task. The main structure is shown in \Cref{fig_vmuviawatermark}b. 

We fixed the minibatch size to 16 for all experiments. The learning rate is set to $\eta=0.001$ for training the model on MNIST and $\eta=0.005$ for training on CIFAR-10 and CelebA. The number of training epochs is set to $E=50$. All models are implemented using PyTorch, and experiments are conducted on NVIDIA Quadro RTX 6000 GPUs.




\vspace{1mm}
\noindent
\textbf{Evaluation Metrics.}
The primary goal of unlearning verification is to enable users to verify whether their data has been removed from the model after unlearning. Based on the requirements for unlearning verification outlined in \Cref{defining_requirements}, we summarize three metrics to evaluate the effectiveness of verification and preservation of model utility: verifiability, unambiguity, and model accuracy. Additionally, we use running time to evaluate the efficiency of the proposed method. These four metrics are summarized as follows. 
\begin{itemize} [itemsep=0pt, parsep=0pt, leftmargin=*]
	\item \textbf{Verifiability.} Verifiability assesses whether the seed can be successfully verified. This can be calculated using the correct classification rate of $\mathcal{V}$ based on the seed-embedded data as:  $\textbf{Verifiability} =    \frac{1}{m} \sum_i^m  \mathbb{I}(\mathcal{V}(M_{\texttt{S}}, x_{i,s_i})=1 ),$
	where $\mathbb{I}$ is the indicator function that equals 1 when its argument is true ($\mathcal{V}(M_{\texttt{S}}, x_{i,s_i})=1$) and 0 otherwise. Here, $x_{i,s_i} \in D_{u,s}$ represents the seed-embedded data, and $m$ is the size of $D_{u,s}$. 
	\item \textbf{Unambiguity.} Unambiguity ensures that other random seeds or the absence of seed are not mistakenly verified as the user's specific seed. This is calculated using the $\mathcal{V}$ results based on data with other seeds as: $\textbf{Unambiguity} =   \frac{1}{m}  \sum_i^m  \mathbb{I} ( \mathcal{V}(M_{\texttt{S}}, x_{i,s_i'})=0 ),$
	where $\mathbb{I}$ is the indicator function that equals 1 when its argument is true ($\mathcal{V}(M_{\texttt{S}}, x_{i,s_i'})=0$) and 0 otherwise. 
	$x_{i,s_i'} \in D_{u,s'}$ and $D_{u,s'}$ is the data generated with other seeds. 
	\item \textbf{Accuracy.} The model accuracy evaluates functionality-preserving to test if the verification scheme will influence the primary task. It is calculated using the model accuracy on the test dataset for the primary task. 
	\item  \textbf{Running Time.} We calculate running time by recording the running time used in one batch and multiplying it with the training epochs.
\end{itemize}

\vspace{1mm}
\noindent
\textbf{Compared Verification Method and Tested Unlearning Benchmarks.} 
There are three main data removal verification approaches~\cite{hu2022membership,guo2023verifying,sommer2022athena}, all based on backdooring methods. We directly compare our SMS scheme with the MIB method \cite{hu2022membership}, as it has demonstrated the best verification effectiveness among these approaches. \citeauthor{guo2023verifying} \cite{guo2023verifying} focused extensively on designing invisible backdoor triggers, which, to some extent, diminished the verification ability. It is important to note that backdoor-based methods mix backdoored samples into the training dataset and can only build the connection between backdoored samples and models, only supporting verification of the backdoored samples. In contrast, SMS verifies the unlearning of genuine samples with integrated seed features. {Additionally, we also compare SMS with the membership inference attack verification (MIA) methods used in unlearning evaluation~\cite{zhao2024makes,kurmanji2023towards,chen2021machine}. Note that, in the implementation of MIA, MIB and SMS, the users only have the black-box access to the ML server's model. The users can only access the uploaded unlearning samples and the output of the server's model.
}

We evaluate SMS and the state-of-art method MIB \cite{hu2022membership} based on three mainstream unlearning benchmarks: naive retraining, SISA~\cite{bourtoule2021machine}, and VBU~\cite{nguyen2020variational}. Naive retraining is the gold-standard method for machine unlearning, and the SISA \cite{bourtoule2021machine} method is a representative exact unlearning method, and the VBU is a representative approximate unlearning method proposed to unlearn models trained using variational Bayesian inference. In our SISA implementation, we set $k=5$ disjoint shards and their respective sub-model.


\subsection{Evaluation Overview}

We first present the overall evaluation results on MNIST, CIFAR-10, and CelebA using the four introduced evaluation metrics on models before and after unlearning, as shown in \Cref{tab_total}. Bolded values indicate the best performance among these methods, while red-colored values signify results that are opposite from expectations. We primarily compare our method with the state-of-the-art unlearning verification method MIB~\cite{hu2022membership}. To demonstrate functionality preservation, we also record the performance of training the model solely for the primary task, referred to as Non-Verification (abbreviated as Non-Verif.). One important variable that significantly impacts verification performance is the seed-embedded sample rate ($\it{SSR}$) of the entire dataset. A larger $\it{SSR}$ facilitates seed embedding, similar to the backdoored sample rate in MIB. For convenience, we uniformly use the $\it{SSR}$ sign in both SMS and MIB methods. In the overview evaluation, we set $\it{SSR} = 0.6\%$.

\begin{table}[t]
	\scriptsize
	\caption{ {The Verification Effect of Unlearning Genuine Sample of MIA, MIB and SMS}  }
	\label{mia_figure}
	\resizebox{\linewidth}{!}{
		\setlength\tabcolsep{2.5pt}
		\begin{tabular}{c|cc|cc|cc|}
			\toprule
			\multirow{2}{*} { \makecell[c]{\textbf{Datasets}  } } &  \multicolumn{2}{c} {MIA}   &   \multicolumn{2}{c} {MIB}  &   \multicolumn{2}{c} {SMS}    \\
			\cmidrule(r){2-3} \cmidrule(r){4-5} 		\cmidrule(r){6-7} 		
			& Original
			& Unlearned
			& Original
			& Unlearned
			& Original &Unlearned	 \\
			\midrule
			MNIST        & 63.86\%      	 &57.61\%	   & 100\% & 100\%                  & 100\% 	  &0\%    \\
			CIFAR10        & 77.43\%      	 &68.43\%	   & 100\% & 100\%                & 100\%	  &0\%    \\
			CelebA	& 58.37\%     	 &53.93\%	   &93.27\%  & 92.89\%                & 95.38\% 	 &3.41\%     \\
			\bottomrule
	\end{tabular}}
\end{table}

\begin{figure*}[t]
	\centering
	\subfloat[\footnotesize Model Accuracy]{ \label{fig_mnistaccmsrcurve} \rotatebox{90}{ \hspace{8mm}	\scriptsize{ On MNIST} }
		\includegraphics[scale=0.25]{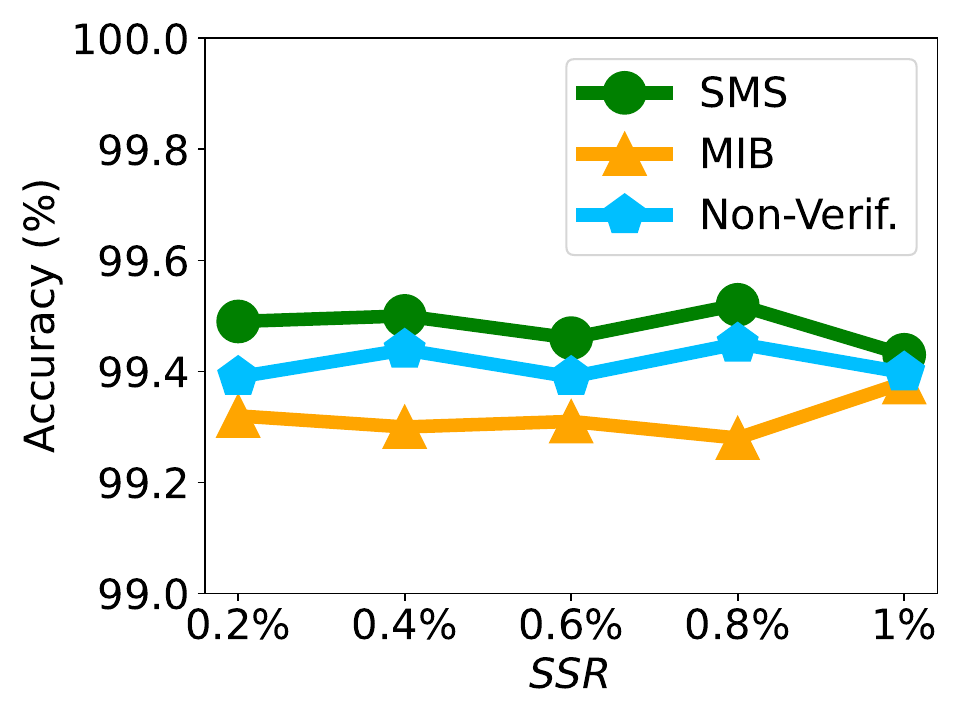}
	}
	\subfloat[\footnotesize Verifiability after Unlearning]{ \label{fig_mnistvsrmsrcurve}
		\includegraphics[scale=0.25]{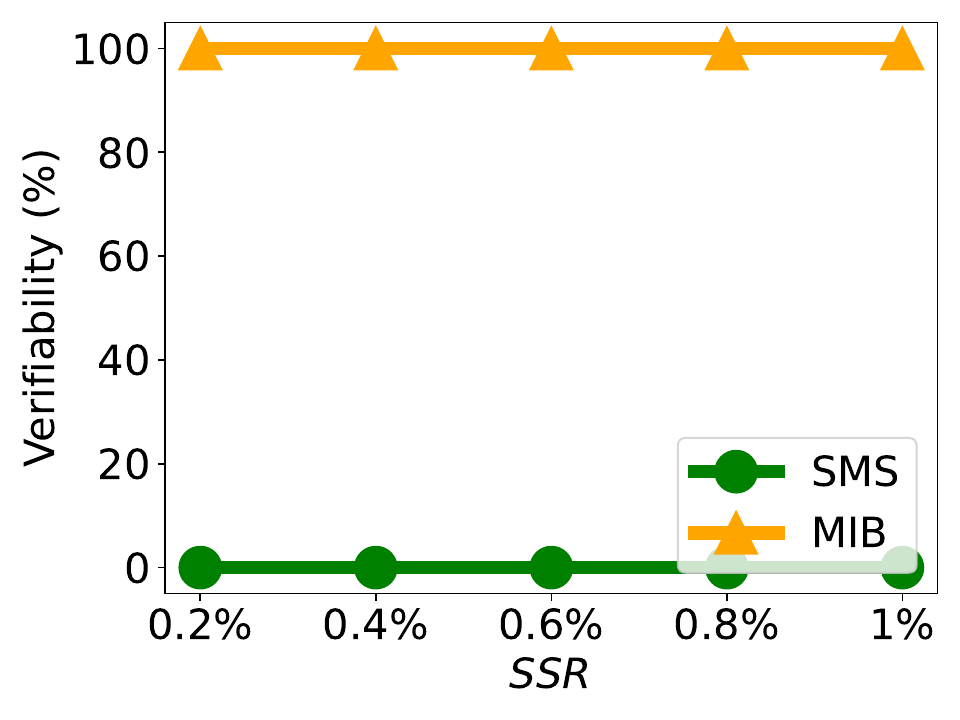}
	}
	\subfloat[\footnotesize Unambiguity]{\label{fig_mnistwvrmsrcurve}
		\includegraphics[scale=0.25]{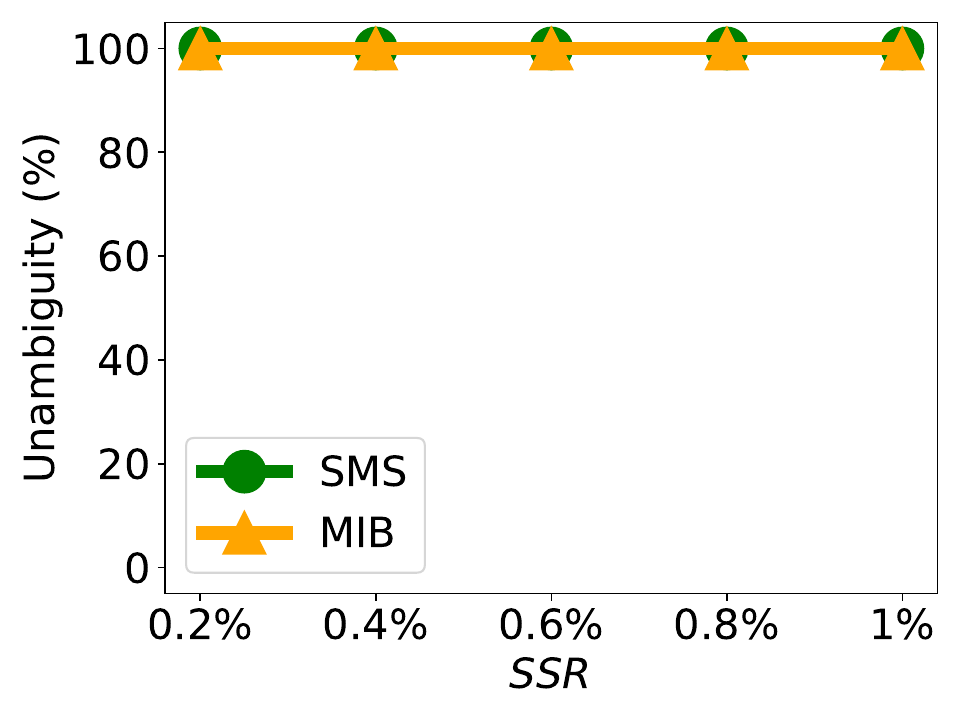}
	}
	\subfloat[\footnotesize Running Time]{	\label{fig_mnistrtmsrbar}
		\includegraphics[scale=0.25]{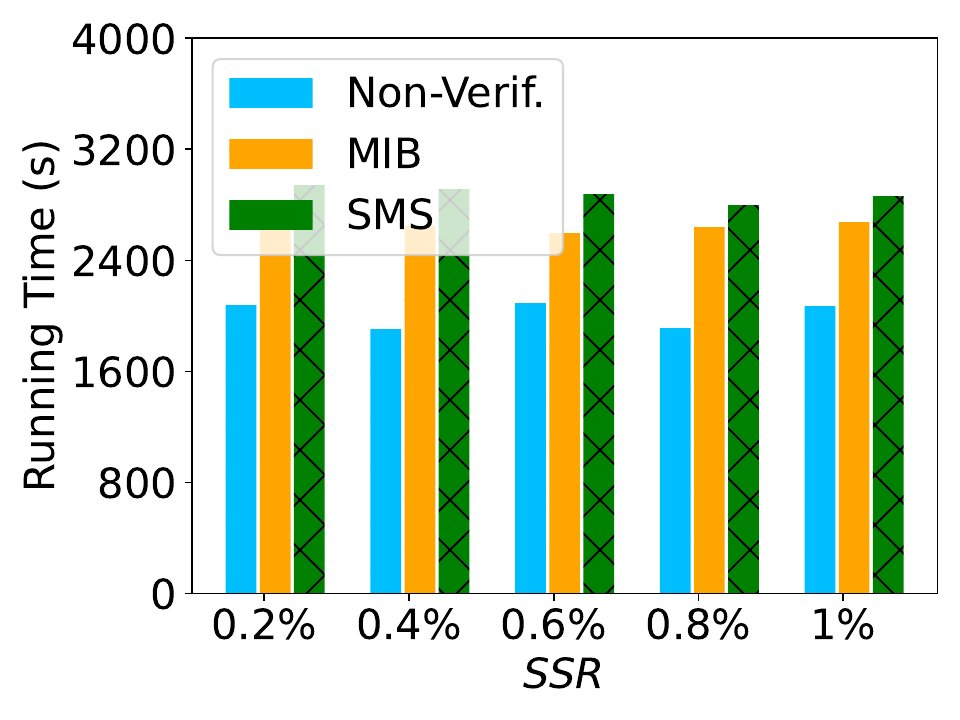}
	}
	\\
	\vspace{-2mm}
	\subfloat[\footnotesize Model Accuracy]{ \label{fig_cifar10accmsrcurve} \rotatebox{90}{ \hspace{8mm}	\scriptsize{ On CIFAR10} }
		\includegraphics[scale=0.25]{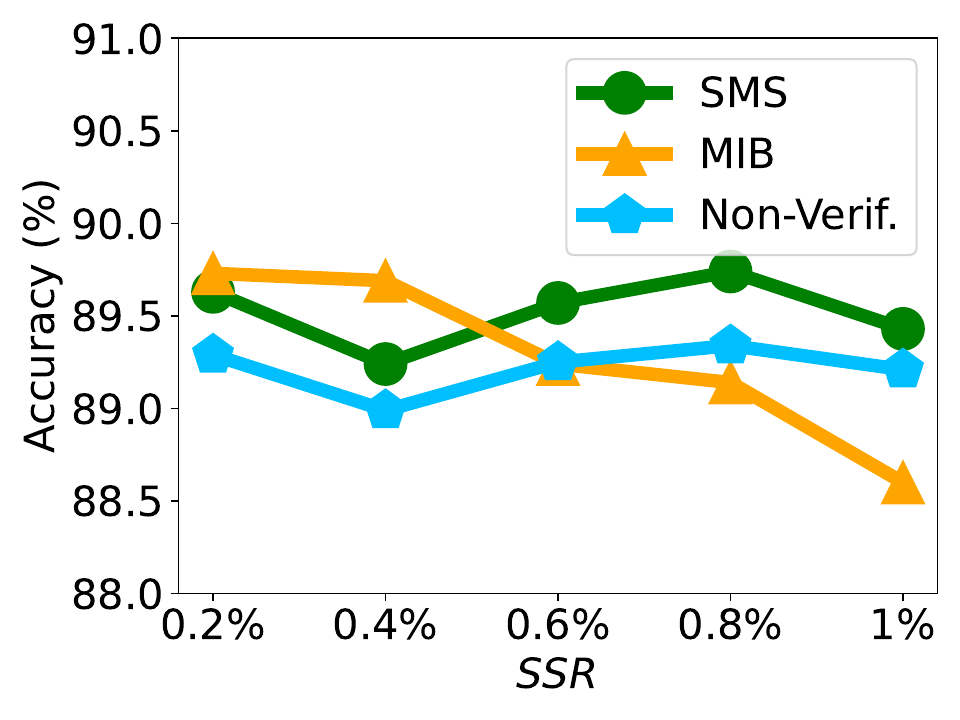}
	}
	\subfloat[\footnotesize Verifiability after Unlearning]{	\label{fig_cifar10vsrmsrcurve}
		\includegraphics[scale=0.25]{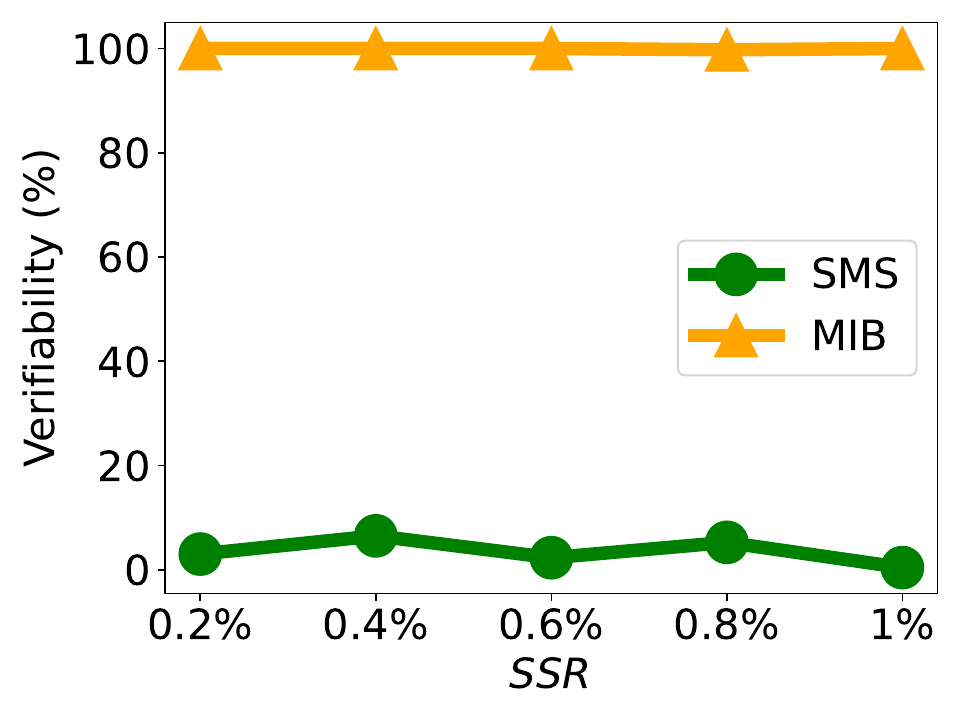}
	}
	\subfloat[\footnotesize Unambiguity]{ 	\label{fig_cifar10wvrmsrcurve}
		\includegraphics[scale=0.25]{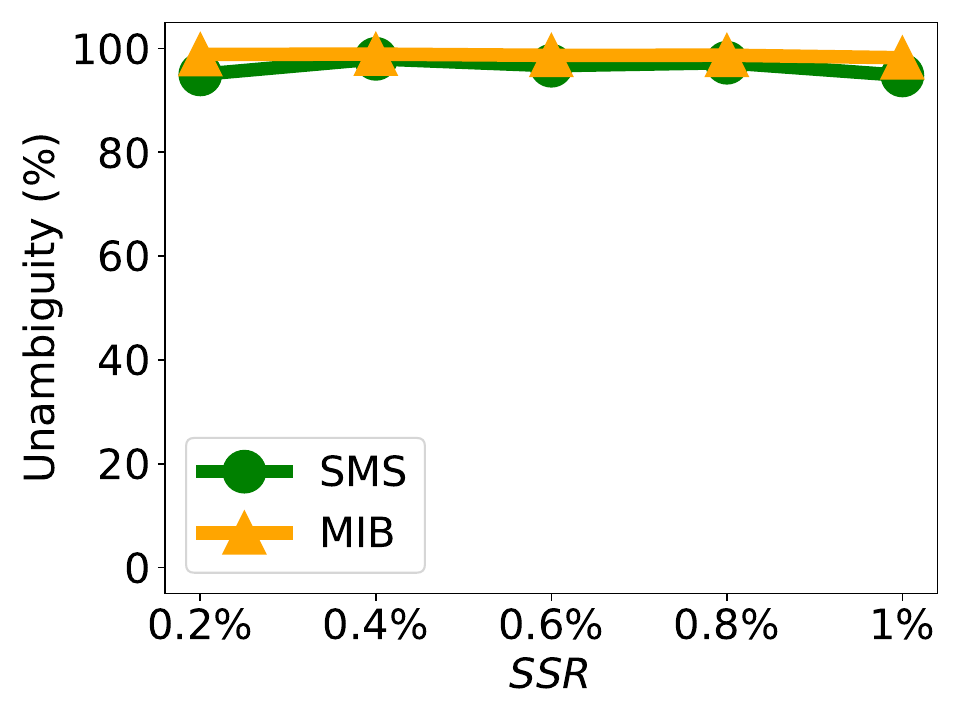}
	}
	\subfloat[\footnotesize Running Time]{	\label{fig_cifar10rtmsrbar}
		\includegraphics[scale=0.25]{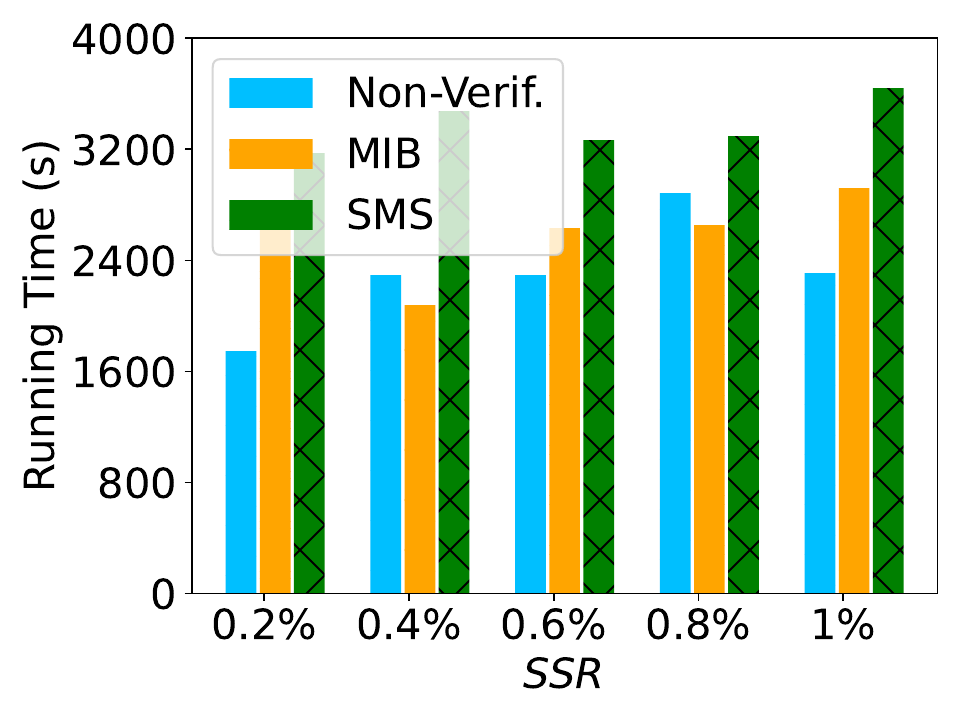}
	}
	\caption{Evaluations of the impact of different ${\it SSR}$ on MNIST and CIFAR10} 
	\label{evaluation_of_msr} 
	\vspace{-2mm}
\end{figure*}

In the evaluation of methods before unlearning, both SMS and MIB achieve effective verifiability and unambiguity. However, MIB verification applies only to backdoored samples, which are distinct from the user's genuine data. Therefore, MIB's verifiability pertains solely to the backdoored samples, not the true genuine samples. In contrast, SMS integrates the seed as a feature within the genuine data, making the verifiability of SMS's seed equivalent to the verifiability of the features of the genuine data, which is more reliable than MIB. With the above benefits of SMS, the cost is that SMS incurs a higher running time due to the additional self-supervised model seeding task, as indicated by the running time metric in the upper half of \Cref{tab_total}.

In the evaluation of methods after unlearning, only the verifiability of SMS reduces to $0$, indicating that SMS can detect the disappearance of the seed from the genuine erased sample after unlearning. {Here, we implement unlearning using both the retraining and the widely adopted exact unlearning method, SISA~\cite{bourtoule2021machine}, and request the removal of several genuine samples rather than the user's entire dataset. SMS successfully verifies that unlearning has been performed (both retraining and SISA) and demonstrates superior functionality preservation across all three datasets, as shown in \Cref{tab_total}.}

Since the membership inference attack methods (MIA)~\cite{kurmanji2023towards,chen2021machine} do not have the unambiguity metrics, we only show the verification effect comparison of MIA, MIB, and SMS on \Cref{mia_figure}. We set $\it{SSR} = 0.6\%$, and the unlearning method is SISA. The verifiability of MIA does drop (such as 63.86\% of the original model and 57.61\% of the unlearned model on MNIST), which indicates the unlearning effect, but is not as effective and obvious as our SMS method.

\subsection{Impact of the Seed-embedded Samples Rate (SSR)} \label{MSR_exp}

In this paper, we examine two main variables that significantly impact verification effectiveness and model training consumption: the seed-embedded sample rate ($\it{SSR}$) and the seed embedding rate ($\it{SER}$, also denoted as $v$ in \Cref{eq_embed1}). We first evaluate the impact of $\it{SSR}$, with the main results on MNIST and CIFAR-10 presented in \Cref{evaluation_of_msr}. When evaluating the impact of $\it{SSR}$, we keep $\it{SER}$ fixed, and vice versa.
We evaluate SMS, MIB and Non-Verif. from four metrics: model accuracy to evaluate the functionality preservation in \Cref{fig_mnistaccmsrcurve,fig_cifar10accmsrcurve}; verifiability after unlearning in \Cref{fig_mnistvsrmsrcurve,fig_cifar10vsrmsrcurve}; unambiguity in \Cref{fig_mnistwvrmsrcurve,fig_cifar10wvrmsrcurve}; and training time in \Cref{fig_mnistrtmsrbar,fig_cifar10rtmsrbar}. 

The effectiveness of SMS and MIB encompasses model utility preservation, verifiability after unlearning, and unambiguity. Functionality preservation requires verification schemes to ensure the model's utility for the primary task. As shown in \Cref{fig_mnistaccmsrcurve,fig_cifar10accmsrcurve}, SMS improves the model's utility for the primary task on both MNIST and CIFAR-10. This improvement is due to SMS's self-supervised model seeding task, which embeds all sample information into the model, benefiting the primary task. In contrast, MIB changes the labels of backdoored samples, thus harming the model's primary task utility. This issue worsens as $\it{SSR}$ increases. We have not included figures for verifiability before unlearning because both SMS and MIB can successfully verify the usage of marked data, although MIB can only verify backdoored samples. For verifiability after unlearning, only SMS can successfully verify the unlearning of genuine samples, while MIB fails to do so, as shown in \Cref{fig_mnistvsrmsrcurve,fig_cifar10vsrmsrcurve}. Regarding unambiguity, both SMS and MIB demonstrate high performance on MNIST and CIFAR-10, as shown in \Cref{fig_mnistwvrmsrcurve,fig_cifar10wvrmsrcurve}.

We evaluate efficiency based on training time, as shown in \Cref{fig_mnistrtmsrbar,fig_cifar10rtmsrbar}. Both SMS and MIB increase the training time compared to a model that solely trains for the primary task. 
For SMS, the training time is not closely related to $\it{SSR}$ because the seeds are integrated as features into the data, which does not affect the size of the training dataset. However, for MIB, a larger $\it{SSR}$ means adding more backdoored samples. Since we set a small $\it{SSR}$, the training time only slightly increases as $\it{SSR}$ increases for MIB. SMS consumes more running time than MIB and Non-Verif., due to the additional self-supervised model seeding task required for model seeding joint training.

\subsection{Impact of the Seed Embedding Rate (SER) } \label{MER_exp}

In this section, we evaluate the impact of the seed embedding rate ($\it{SER}$), which represents the extent of changes the seed introduces to the original sample. For ease of comparison with MIB, we use the data embedding methods as defined in \Cref{eq_embed1}, which is also employed in MIB. In this context, $v$ is equivalent to $\it{SER}$. We use $v$ to control the extent of seed embedding into the original sample $x$. In the experiment described in \Cref{tab_watermark_p}, we embed the digit 7 into the bottom right corner of a dog image as an example, and we set different values of $\it{SER}$ to evaluate its impact.

 \begin{figure*}[t]
 	\centering
 	\subfloat[\footnotesize MIB, SISA, user's whole data]{ 	\label{fig_mibsisawholedatacifar10}
 		\includegraphics[scale=0.25]{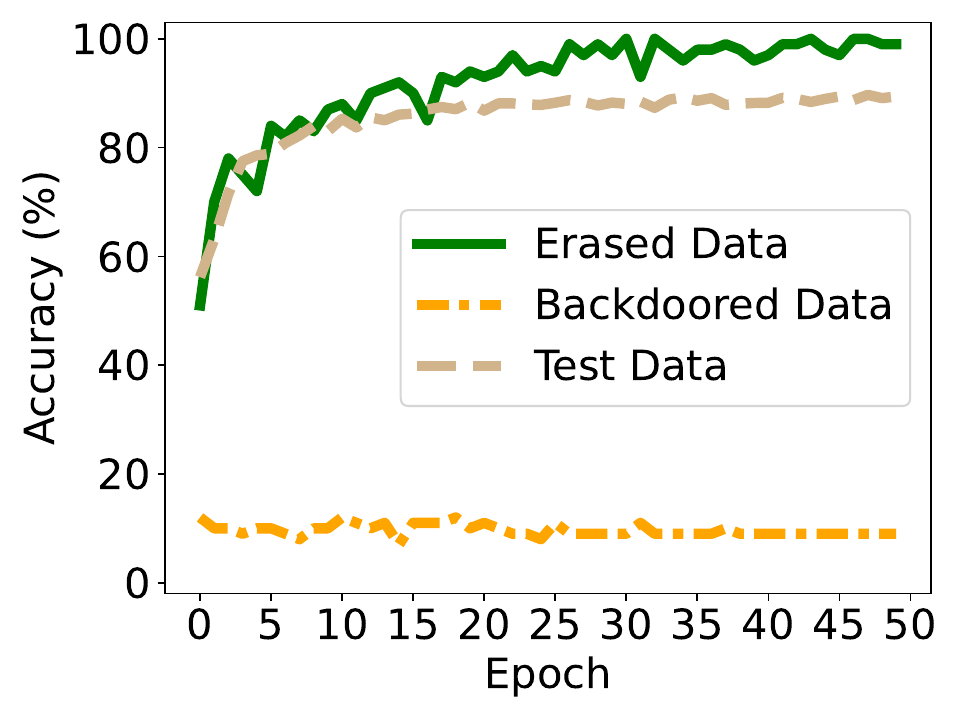}
 	}
 	\subfloat[\footnotesize SMS, SISA, user's whole data]{ 	\label{fig_vmusisawholedatacifar10}
 		\includegraphics[scale=0.25]{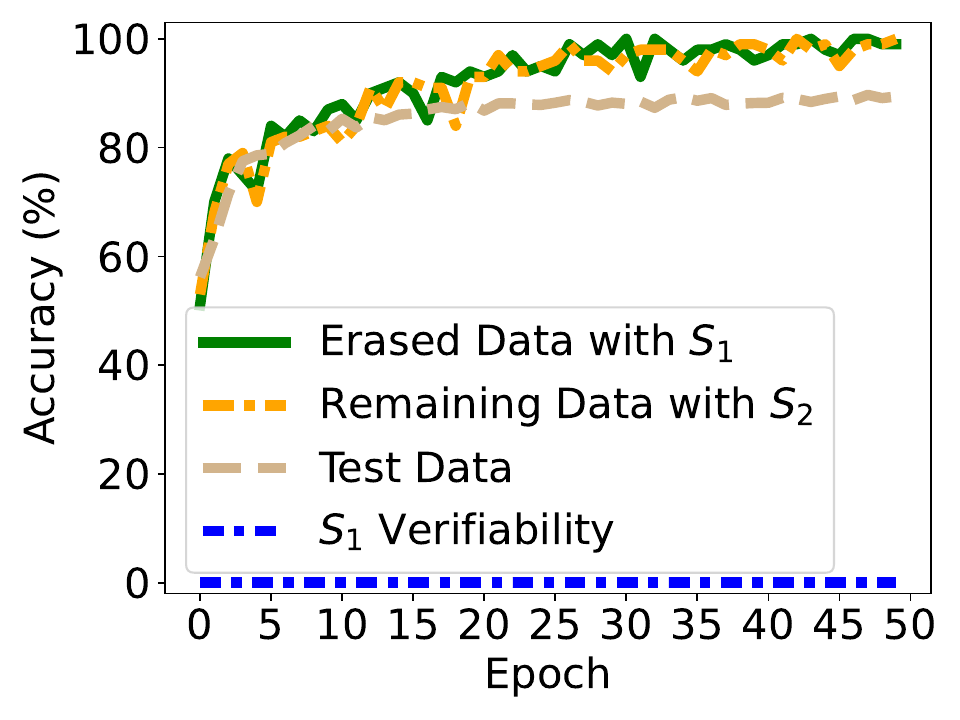}
 	}
 	\subfloat[\footnotesize MIB, VBU, user's whole data]{  	\label{fig_mibvbuwholedatacifar10}
 		\includegraphics[scale=0.25]{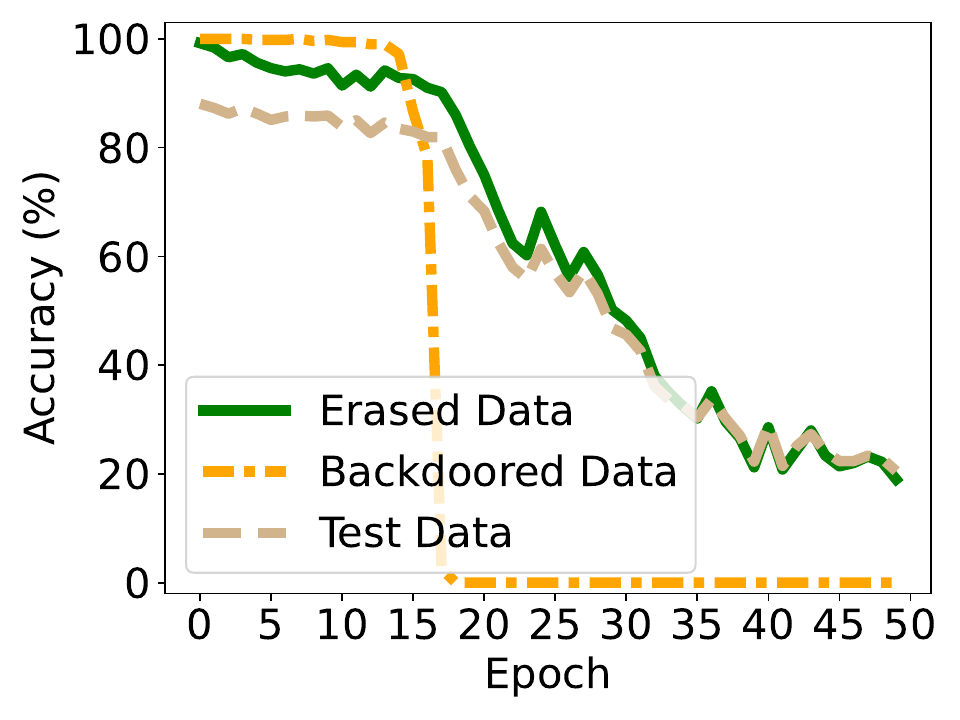}
 	}
 	\subfloat[\footnotesize SMS, VBU, user's whole data]{  	\label{fig_vmuvbuwholedatacifar10}
 		\includegraphics[scale=0.25]{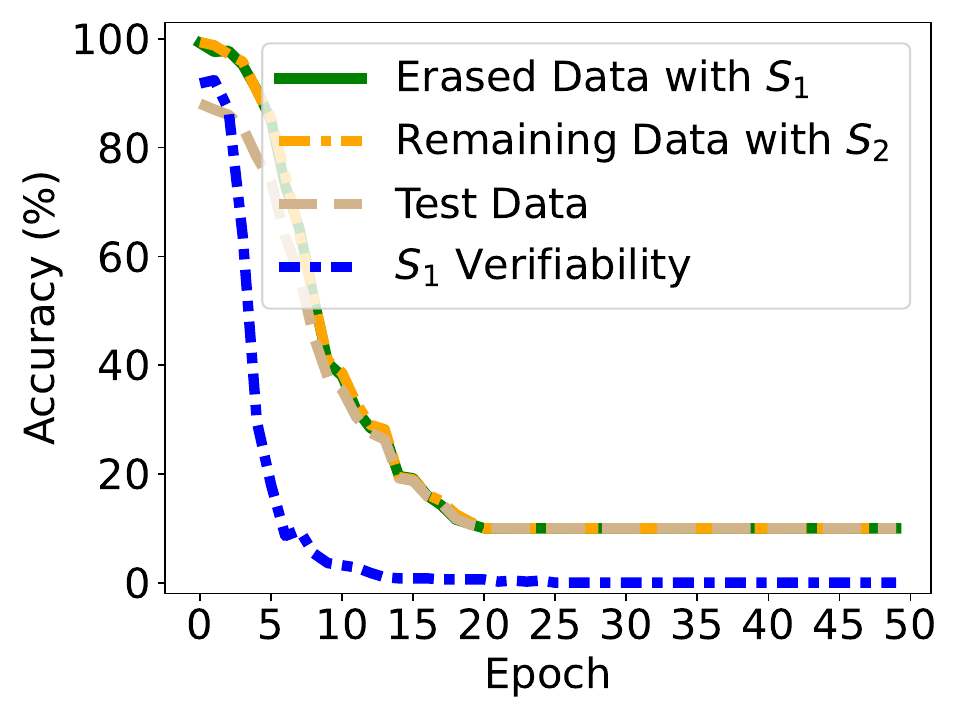}
 	}
 	\\
 	\subfloat[\footnotesize MIB, SISA, erased samples]{	\label{fig_mibsisaeraseddatacifar10}
 		\includegraphics[scale=0.25]{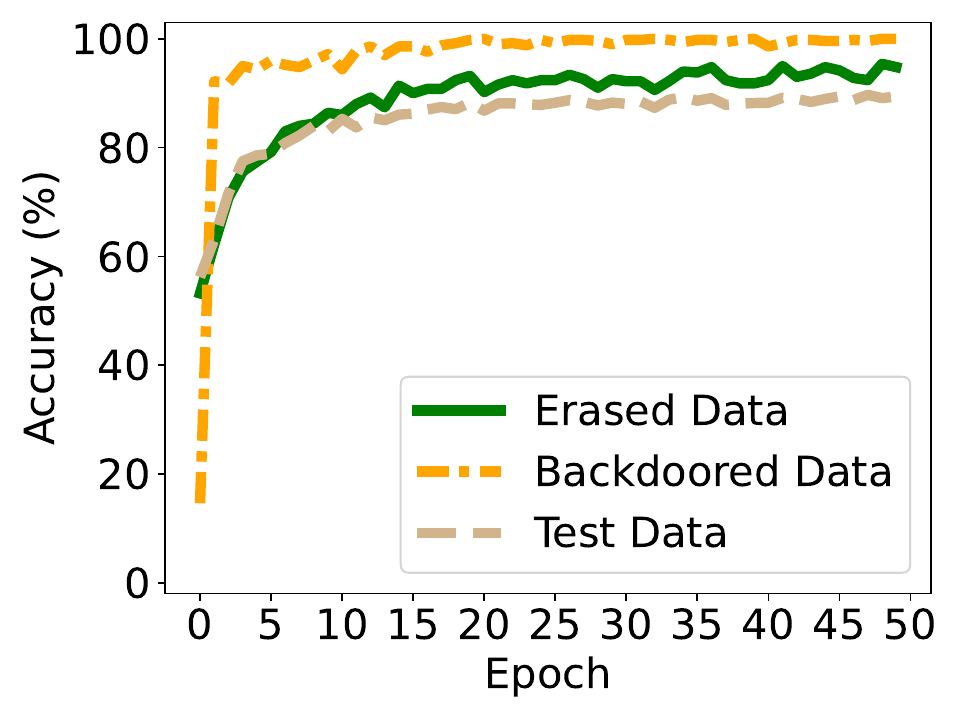}
 	}
 	\subfloat[\footnotesize SMS, SISA, erased samples]{ 	\label{fig_vmusisaeraseddatacifar10}
 		\includegraphics[scale=0.25]{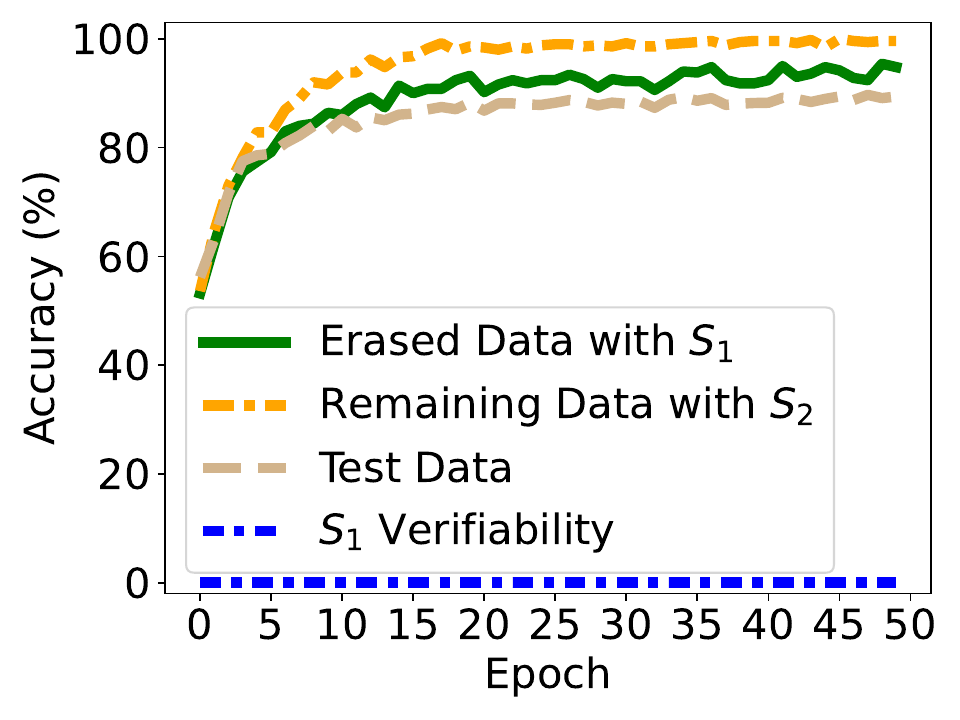}
 	}
 	\subfloat[\footnotesize MIB, VBU, erased samples]{		\label{fig_mibvbueraseddatacifar10}
 		\includegraphics[scale=0.25]{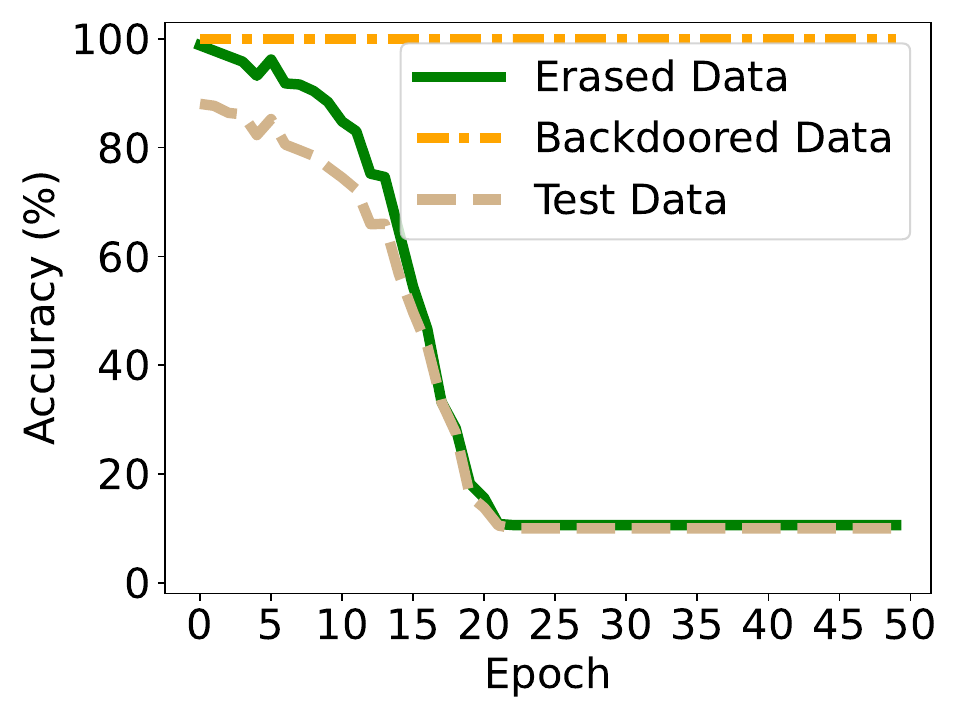}
 	}
 	\subfloat[\footnotesize SMS, VBU, erased samples]{	\label{fig_vmuvbueraseddatacifar10}
 		\includegraphics[scale=0.25]{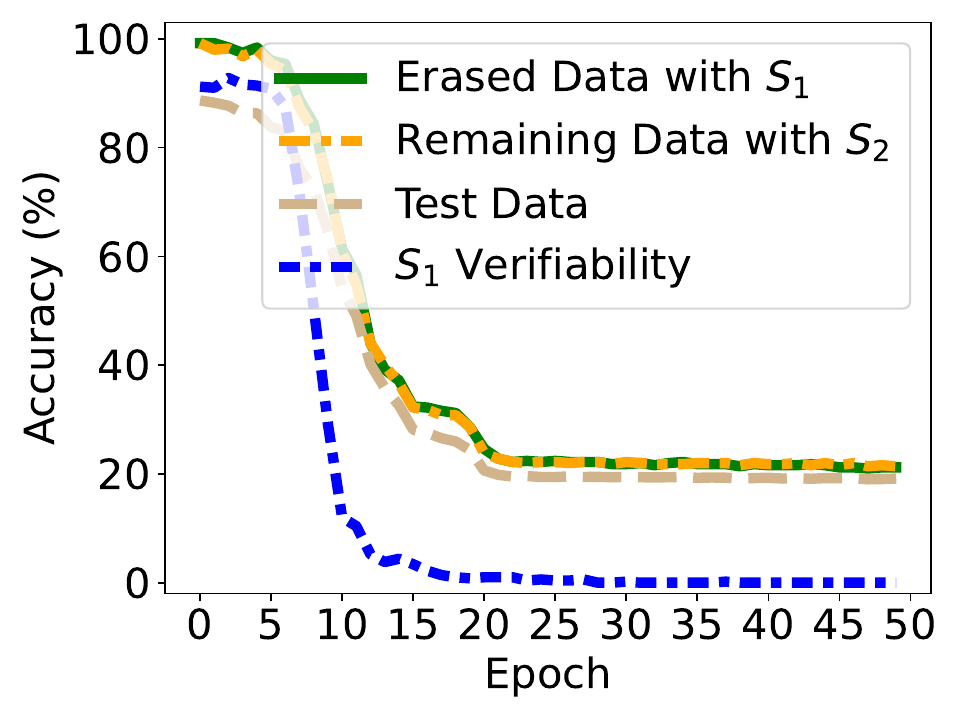}
 	}
 	\caption{Evaluations of unlearning verification performance of SMS and MIB in different unlearning scenarios using different unlearning methods. $S_1$ and $S_2$ are two kinds of seeds added to user-specified samples, $S_1$ for the erased samples and $S_2$ for other samples of his other data.} 
 	\label{evaluation_of_unlearning_ver_p} 
 \end{figure*}

\begin{table}[t]
	\scriptsize
	\caption{Impact of the ${\it SER}$ ($v$) on CIFAR10}
	\label{tab_watermark_p}
	\resizebox{\linewidth}{!}{
		\setlength\tabcolsep{2.5pt}
		\begin{tabular}{cccccc}
			\toprule
			\multirow{2}{*} { \makecell[c]{\textbf{Evaluation} \textbf{Metrics} } } & ${\it SER}$ = 0.2  & ${\it SER}$ = 0.4   & ${\it SER}$ = 0.6  & ${\it SER}$ = 0.8 & ${\it SER}$ = 1   \\
			\cmidrule(r){2-6} 		
			& \adjustbox{valign=b}{
				\includegraphics[scale=0.11]{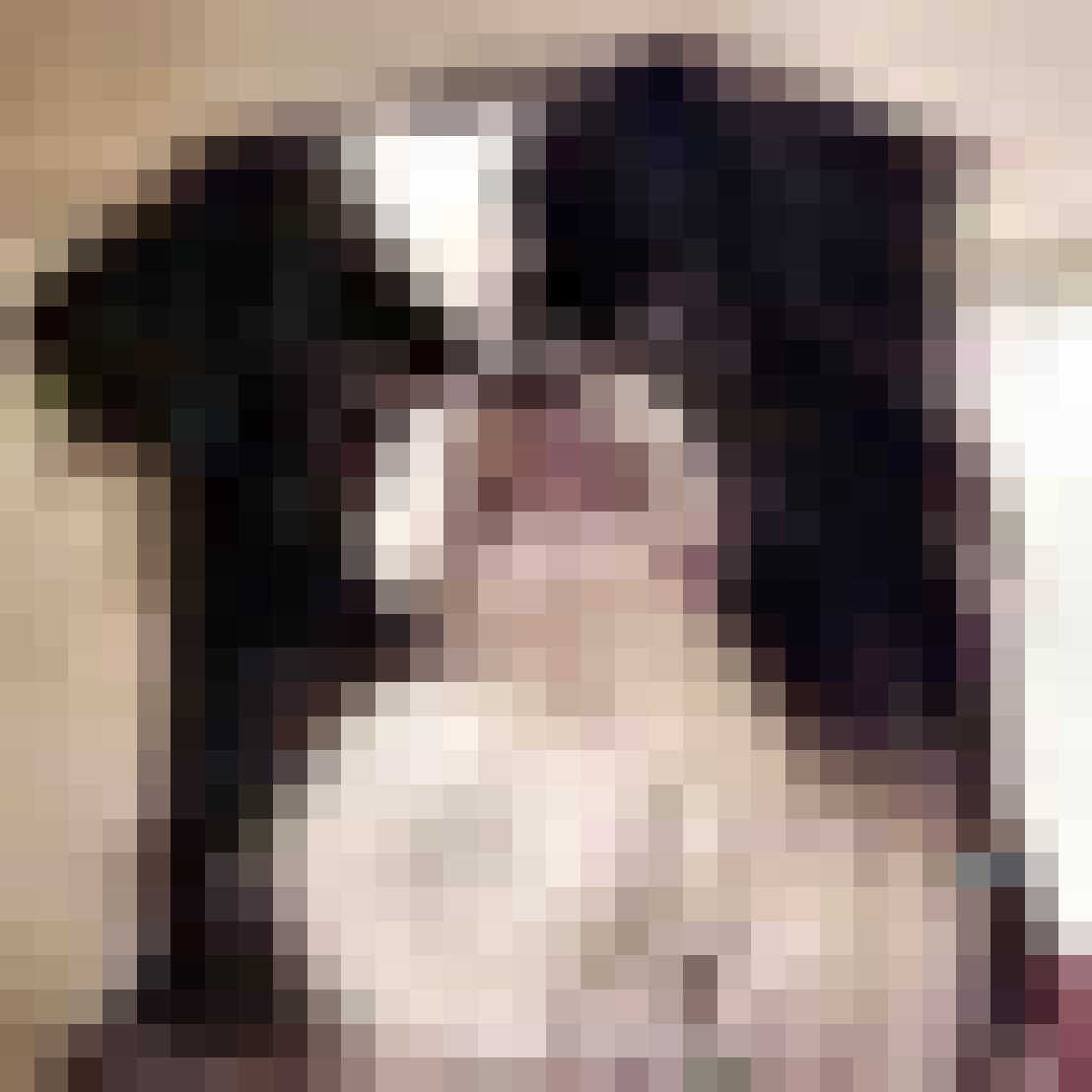} 
			}
			& \adjustbox{valign=b}{
				\includegraphics[scale=0.11]{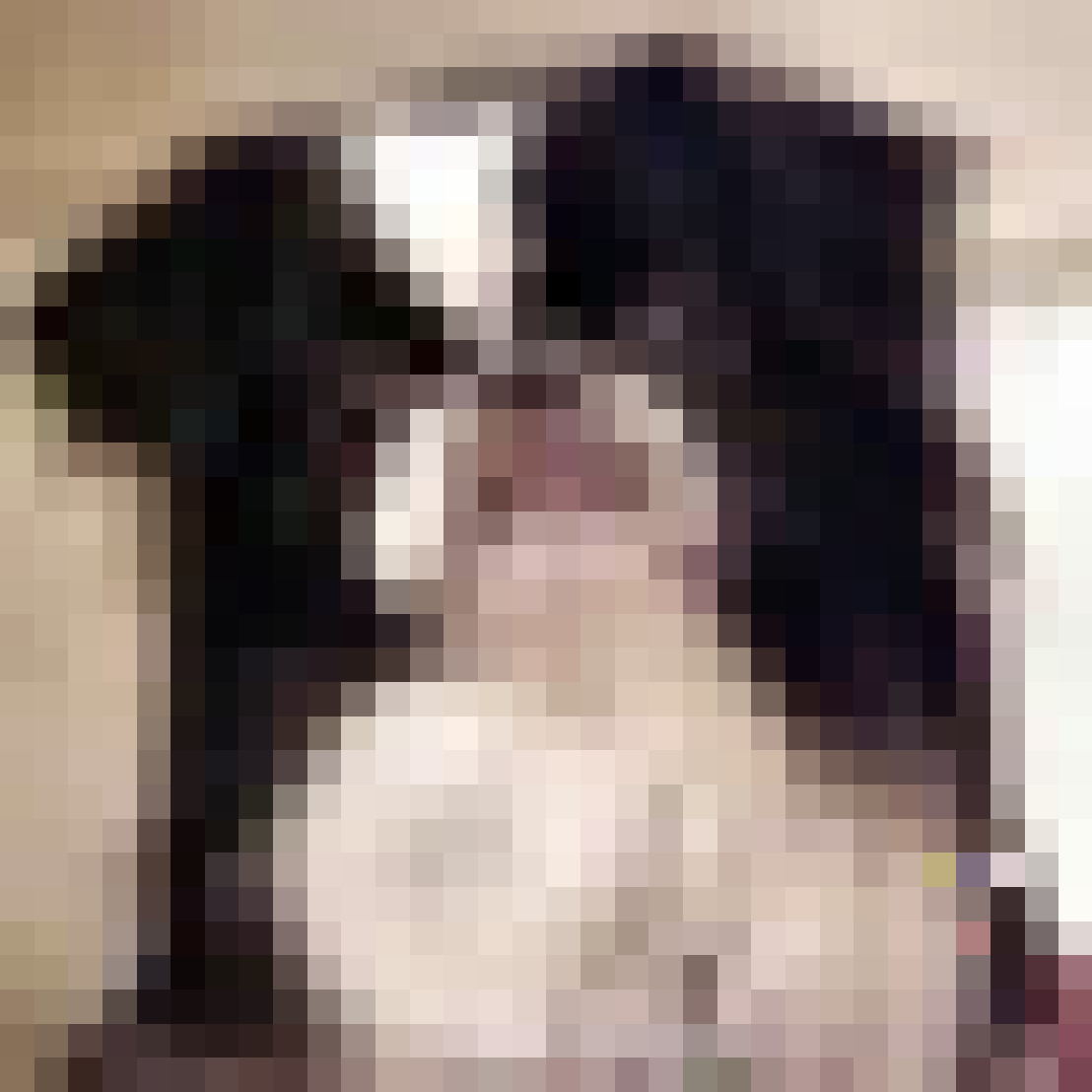} 
			}
			& \adjustbox{valign=b}{
				\includegraphics[scale=0.11]{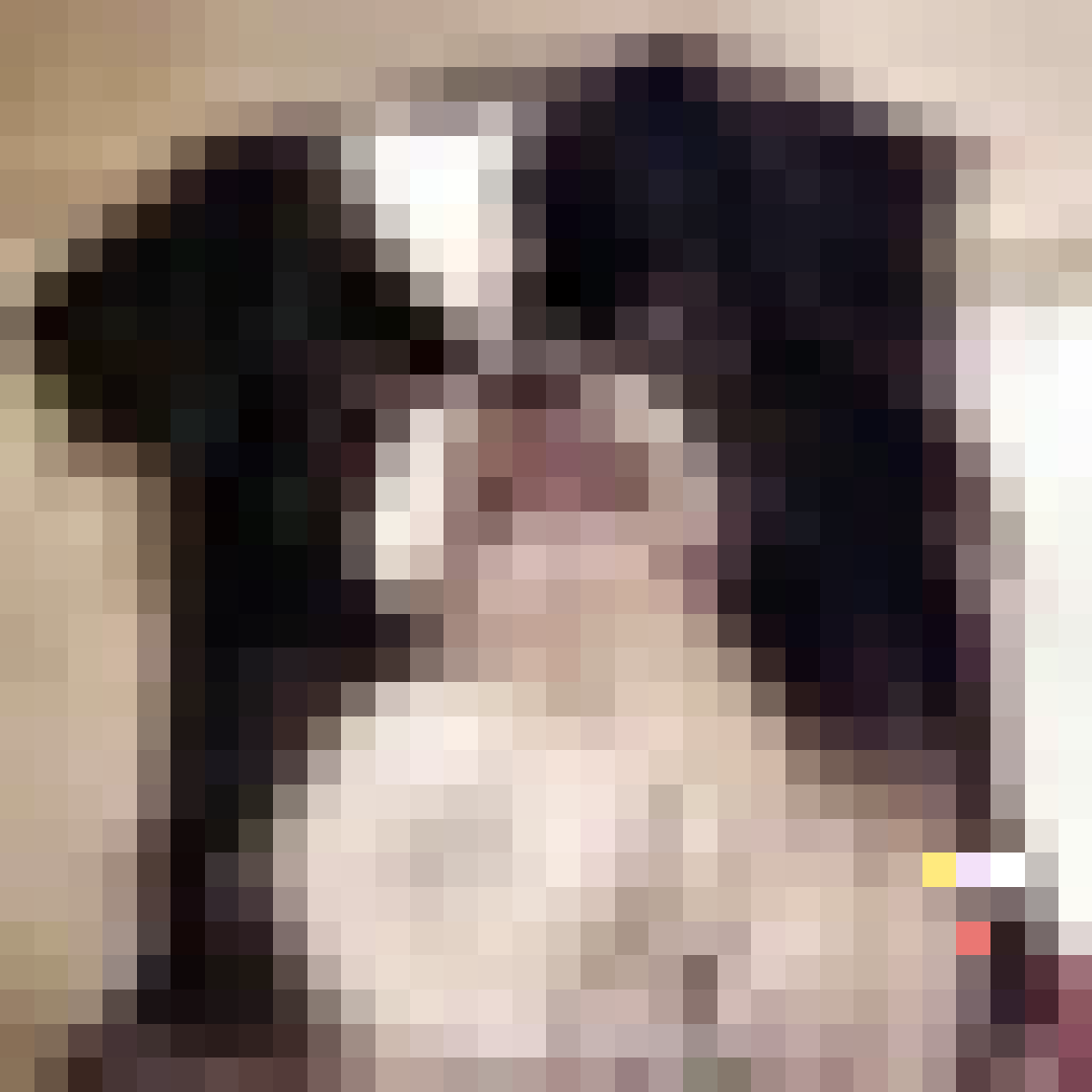} 
			}	 
			&\adjustbox{valign=b}{
				\includegraphics[scale=0.11]{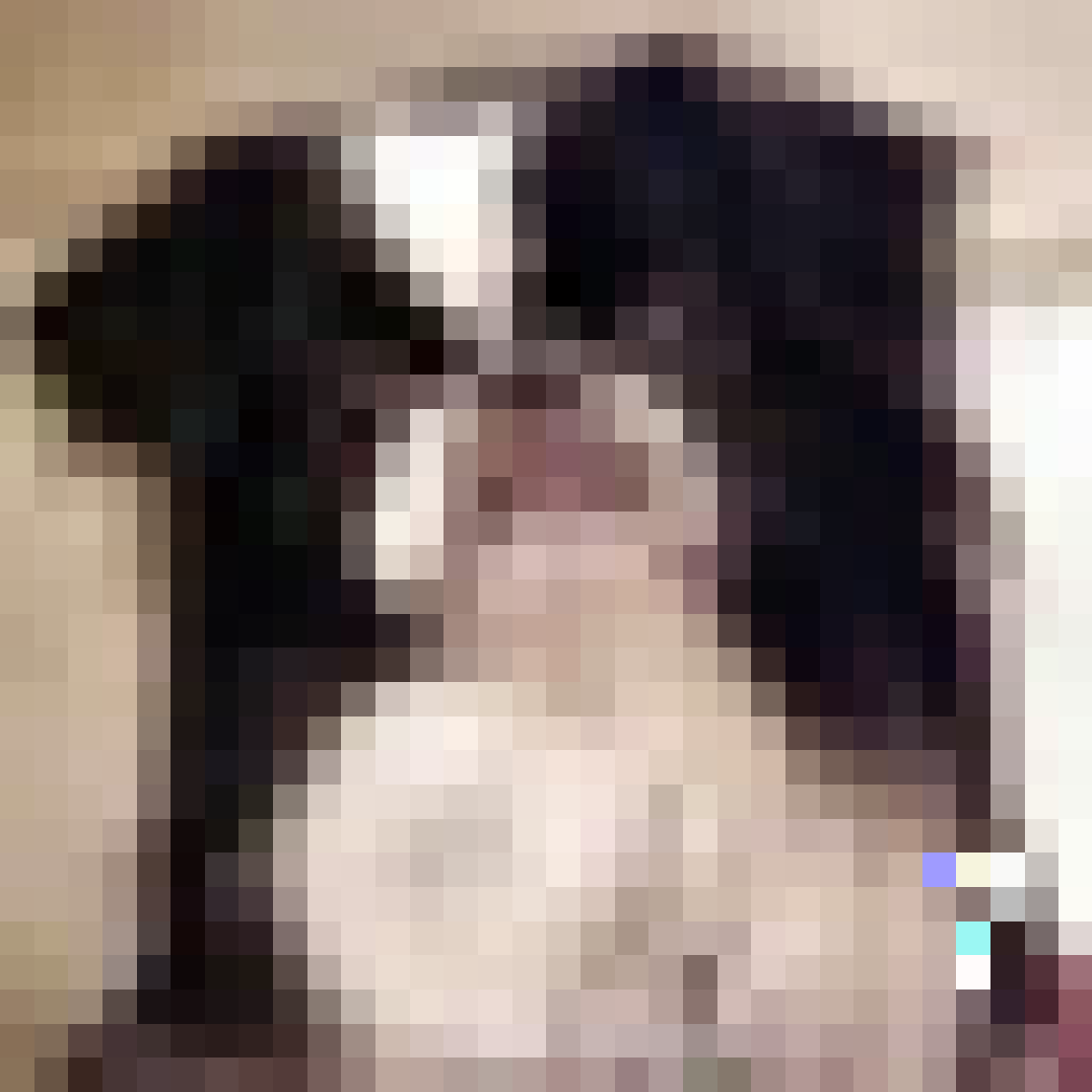} 
			}
			& \adjustbox{valign=b}{
				\includegraphics[scale=0.11]{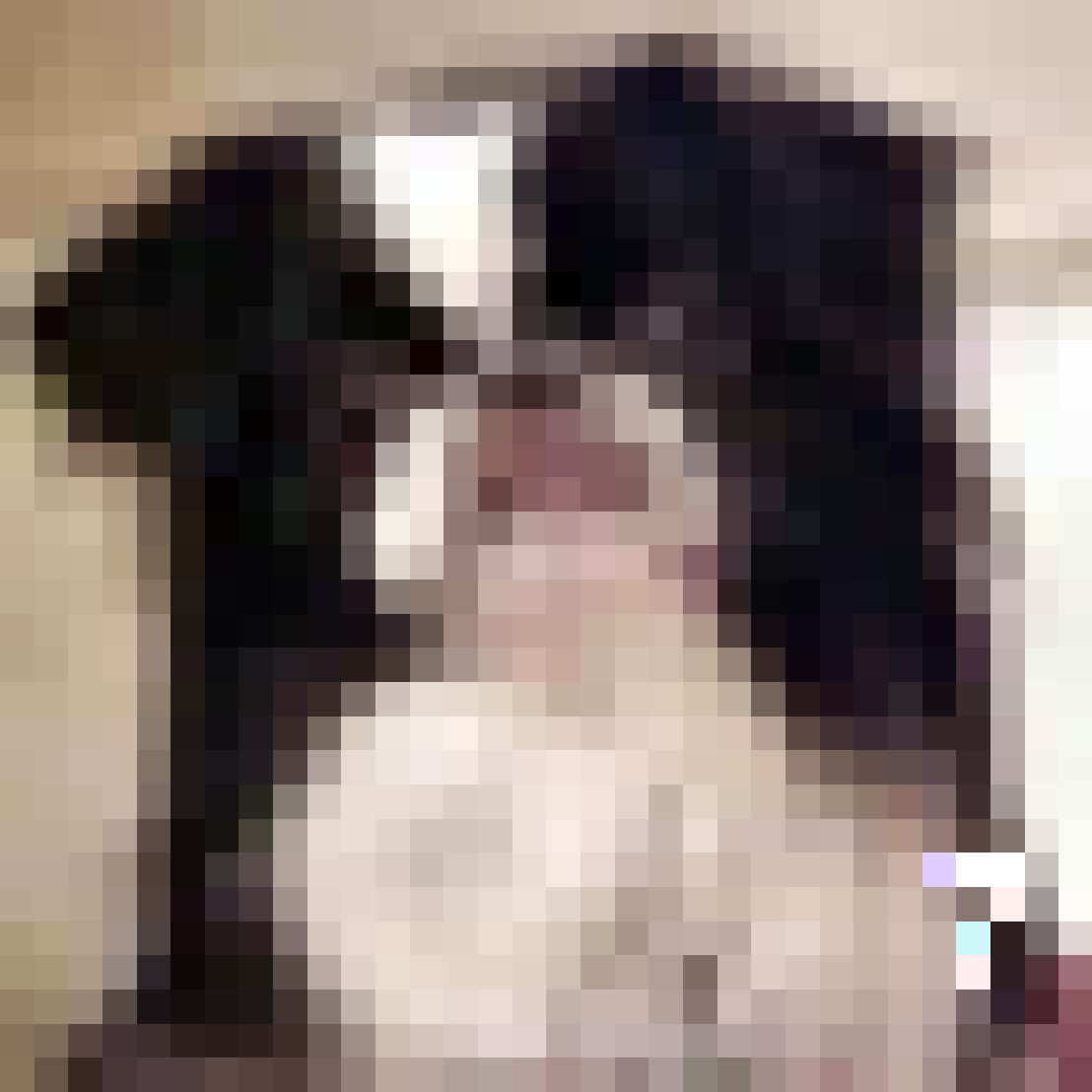} 
			} 	 \\
			\midrule
			Model acc. (MIB)        & 88.63\%      	 &88.76\%	   & 88.69\% & 88.97\%                  & 89.24\% 	     \\
			Verifiability (MIB)        & 97.67\%      	 &98.02\%	   & 99.78\% & 100\%                & 100\%	     \\
			Unambiguity (MIB)		& 98.43\%     	 &98.46\%	   &98.54\%  & 98.39\%                & 98.67\% 	     \\
			Model acc. (SMS)        & 89.62\%      	 &89.55\%	   & 89.85\% & 89.48\%                 & 89.57\% 	     \\
			Verifiability (SMS)        & 65.67\%      	 &86.33\%	   & 85.00\%  & 92.67\%                 & 97.67\%	     \\
			Unambiguity (SMS)		& 87.00\%      	 &91.67\%	   & 97.33\% & 98.00\%               & 96.67\%	     \\
			\bottomrule
	\end{tabular}}
	\vspace{-1mm}
\end{table}

We first analyze the results of MIB~\cite{hu2022membership}. Generally, a higher number of backdoored samples (larger $\it{SSR}$) degrades the model's utility for the primary task. This experiment further demonstrates that a lower embedding rate, $\it{SER}$, also negatively impacts model utility. When $\it{SER}=0.2$, the digit 7 is almost invisible, and the model accuracy degrades to $88.63\%$. This occurs because, with fewer backdoor information injections, the backdoored samples become more similar to genuine samples. Two highly similar samples with different labels create contradictions during model learning. Despite this degradation in model utility, MIB achieves good performance in terms of both verifiability and unambiguity on CIFAR10.

When analyzing the impact of $\it{SER}$ of SMS, seeding the model from seed-embedded samples is more challenging than directly backdooring. We need the self-supervised model seeding task to embed the whole sample into the model's latent space. Learning more feature information further enhances the primary task learning, resulting in better model accuracy. Consequently, SMS consistently achieves higher model accuracy compared to MIB, regardless of the $\it{SER}$. However, since seed embedding is more complex than backdooring, the verifiability and unambiguity are suboptimal when $\it{SER}=0.2$. These metrics improve significantly when $\it{SER} \geq 0.4$.

\subsection{Verifications for Different Unlearning Methods} \label{verify_of_diff_unl}

\begin{table}[t]
	\scriptsize
	\caption{Comparison of SMS and MIB in supporting the verification of different unlearning methods}
	\label{tab_verify_unl}
	\resizebox{\linewidth}{!}{
		\setlength\tabcolsep{2.5pt}
		\begin{tabular}{c||cccc}
			\toprule
			\multirow{2}{*} {Evaluation Metrics } & \multirow{2}{*} { Naive Retraining}  & \multirow{2}{*} { SISA \cite{bourtoule2021machine} }  & \multirow{2}{*} { HBU \cite{guo2019certified}}		  		 & \multirow{2}{*} { VBU \cite{nguyen2020variational}	}    \\
			\\
			\midrule
			MIB (User's whole samples)        & $\checkmark$      	 & $\checkmark$	   & $\times$   & $\times$                     \\
			MIB (Several erased samples)    & $\times$   & $\times$   & $\times$        & $\times$       \\
			SMS (User's whole samples)       &  $\checkmark$  &  $\checkmark$  &  $\checkmark$        &  $\checkmark$      \\
			SMS (Several erased samples)  &  $\checkmark$  &  $\checkmark$  &  $\checkmark$       &  $\checkmark$      \\
			\bottomrule
	\end{tabular}}
\end{table}

In this section, we evaluate the performances of SMS and MIB in supporting the verification of different unlearning methods. We first examine the verification capabilities of the two methods across various unlearning scenarios, as detailed in \Cref{tab_verify_unl}. This table includes prevalent unlearning methods along with two common types of unlearning requests.

MIB can only support retraining-based methods for unlearning the user's entire dataset because it must ensure that all backdoored samples are removed. For commonly specified genuine sample erasure, MIB is ineffective because the backdoored samples are distinct from the user's genuine data. {Since backdoored samples and genuine samples constitute separate sub-datasets, the loss on backdoored samples decreases much faster than on genuine samples in approximate unlearning, as shown in \Cref{fig_mnistepochaccdrop}.} Hence, it is difficult to ensure that the disappearance of backdoors indicates the unlearning of genuine samples. In contrast, SMS can support all unlearning scenarios because it integrates the seed as a feature of the genuine samples. The seed-embedded data is not separated from genuine samples and functions as genuine samples to serve the primary task. {When these seeds disappear, it indicates that the features of the genuine samples have been unlearned, which is also indicated in \Cref{fig_mnistepochaccdropwithwatermardis}. }

In \Cref{evaluation_of_unlearning_ver_p}, we present detailed results by showing the changes in primary task accuracy, backdoor verification, and seed verification during the unlearning training process. We conduct experiments using one of the most popular retraining-based methods, SISA~\cite{bourtoule2021machine}, and one of the most popular approximate unlearning methods, VBU~\cite{nguyen2020variational}, on CIFAR10. These results clearly demonstrate that MIB cannot support requests for unlearning genuine specified samples, as shown in \Cref{fig_mibsisaeraseddatacifar10,fig_mibvbueraseddatacifar10}, where the backdoor verification accuracy remains at $100\%$ even after the unlearning methods are executed.

\begin{figure}[t]
	\centering
	\subfloat{ 	\label{fig_mnistepochaccdrop}
		\includegraphics[scale=0.2615]{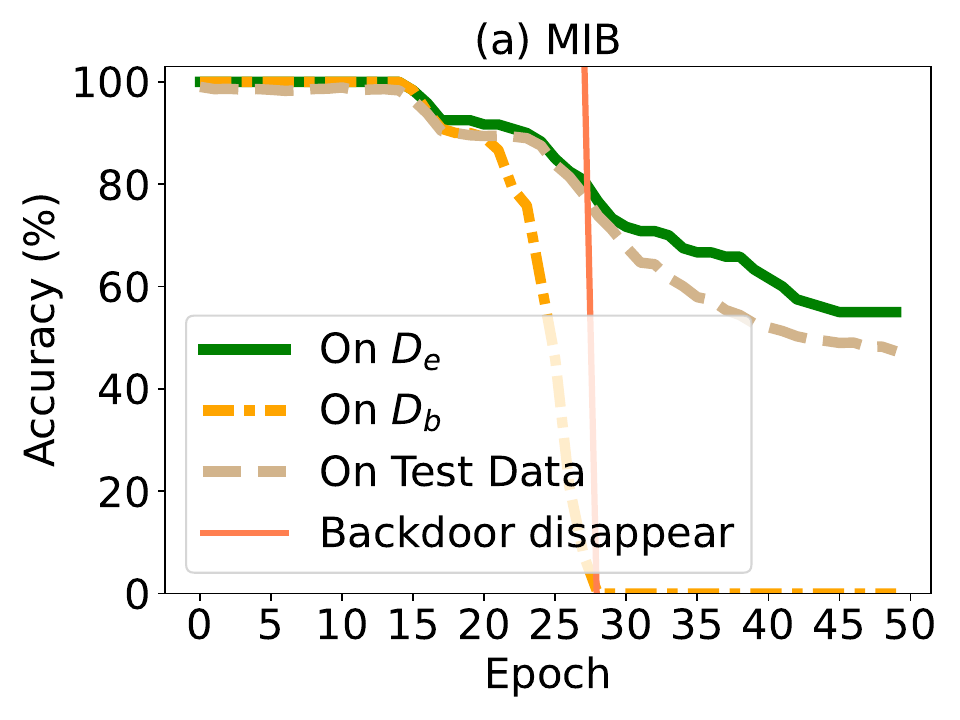}
	}
	\subfloat{ 	\label{fig_mnistepochaccdropwithwatermardis}
		\includegraphics[scale=0.2615]{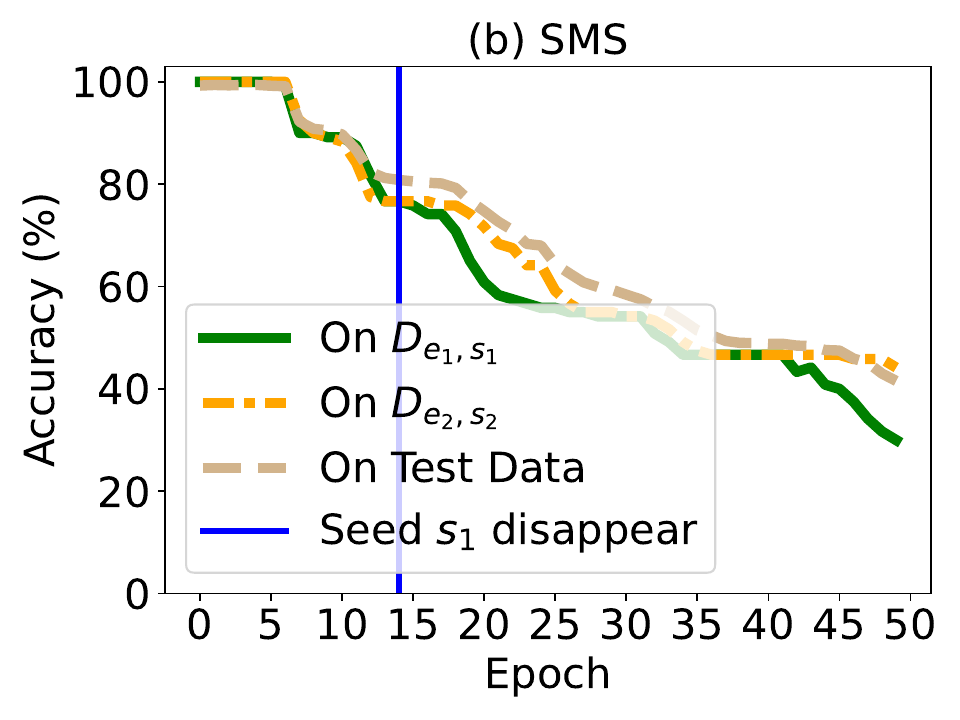} 
	}
\vspace{-2mm}
	\caption{The accuracy changes of SMS and MIB when executing VBU~\cite{nguyen2020variational} on MNIST. (a) The model accuracy diminishes rapidly on backdoored data $D_b$ while it decreases more slowly on erased data $D_e$ and test data. Since $D_e$ and $D_b$ are two independent datasets, the disappearance of the backdoored data $D_b$ does not indicate that the erased data $D_e$ has been unlearned. (b) Seeds $s_1$ and $s_2$ are embedded into the specified erased data $D_{e_1}$ and $D_{e_2}$, respectively. Both $D_{e_1,s_1}$ and $D_{e_2,s_2}$ are genuine data used for primary task training, and hence, they perform similarly during unlearning. The disappearance of seed $s_1$ indicates that the feature information of $D_{e_1,s_1}$ has been unlearned, as $s_1$ is embedded as a minor and invisible feature of $D_{e_1,s_1}$.
	\vspace{-2mm}
} 
	\label{unlearning_opt_process} 
\end{figure}

For SMS, the seed verification accuracy drops to zero in all unlearning methods and scenarios, as shown in \Cref{fig_vmusisawholedatacifar10,fig_vmuvbuwholedatacifar10,fig_vmusisaeraseddatacifar10,fig_vmuvbueraseddatacifar10}. 
In contrast to MIB's performance on approximate unlearning, where the backdoored data is distinct from both the erased data and the remaining data, depicted in \Cref{fig_mibvbuwholedatacifar10,fig_mibvbueraseddatacifar10}, SMS's performance on seed-embedded data aligns with its performance on the remaining data. This consistency is because SMS integrates the seed as a feature within the erased samples. Thus, the disappearance of the seed indicates the removal of the erased samples' features. Specifically, in \Cref{fig_mibvbuwholedatacifar10,fig_mibvbueraseddatacifar10}, the orange-dotted and green lines for MIB represent changes in two datasets: the backdoored data and the genuine erased data. In contrast, for SMS, in \Cref{fig_vmusisawholedatacifar10,fig_vmuvbuwholedatacifar10,fig_vmusisaeraseddatacifar10,fig_vmuvbueraseddatacifar10}, the blue-dotted and green lines indicate verifiability and primary task accuracy for the same dataset, the erased data.


Combined with our experimental results in \Cref{unlearning_opt_process} on MNIST and \Cref{evaluation_of_unlearning_ver_p} on CIFAR-10, we can claim that SMS efficiently supports sample-level genuine data unlearning verification, in both exact and approximate unlearning. However, MIB~\cite{hu2022membership} can only support the user-level exact unlearning verification.

\subsection{Optimization Process of SMS and MIB} \label{opt_process}

In this section, we present the detailed training processes of MIB and SMS to illustrate how these different schemes operate and how they perform on the primary task. The corresponding loss changes of MIB and SMS on MNIST and CIFAR10 are shown in \Cref{training_opt_process}. We use the mean squared error (MSE) loss for model seeding through self-supervised learning. The loss is averaged for each sample, making it slightly smaller than the cross-entropy loss of the primary task.

The key conclusion from these results is that the model seeding joint training method aids in the convergence of the primary task loss. In contrast, backdooring training hinders the primary task loss convergence, causing significant fluctuations, as reflected in the blue and orange dotted lines in \Cref{training_opt_process}. The underlying reason is that SMS's self-supervised model seeding enables the model to learn more feature information of the input, which can further enhance primary task performance. Conversely, MIB changes the label of the original sample, making samples similar to genuine samples but with different labels. When two similar samples have different labels, they create contradictions, hindering the learning of the primary task. The results shown in \Cref{fig_mnistepochlosswhentrianing,fig_cifar10epochlosswhentrianing} confirm our analysis.

\begin{figure}[t]
	\centering
	\subfloat{ 	\label{fig_mnistepochlosswhentrianing} \hspace{-3mm}
		\includegraphics[scale=0.26]{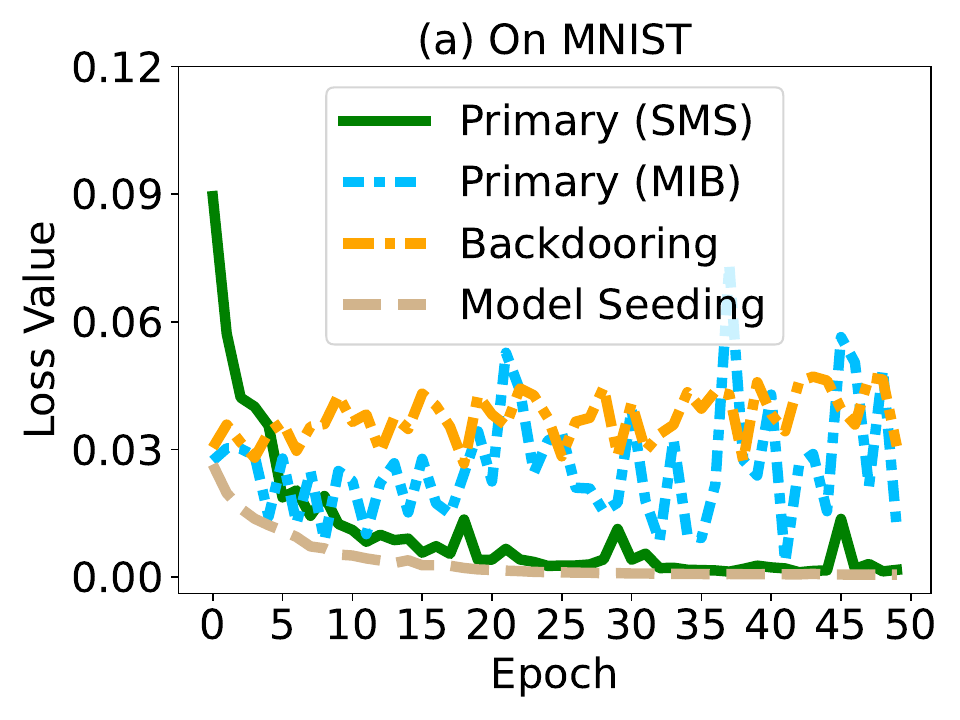}  }
	\subfloat{ 	\label{fig_cifar10epochlosswhentrianing}
		\includegraphics[scale=0.26]{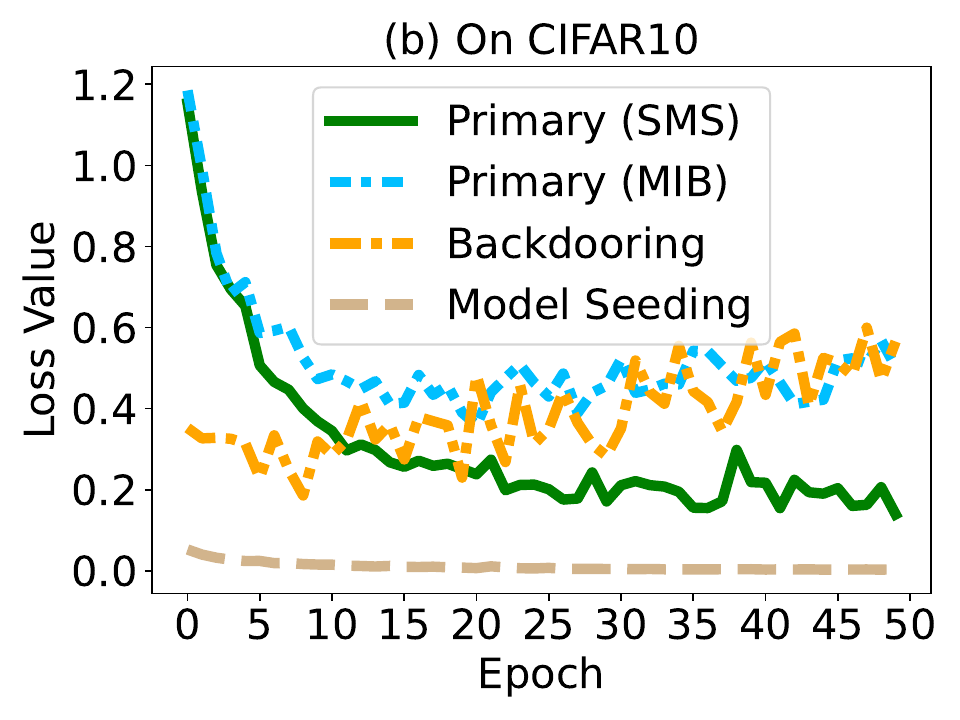}  }
	\vspace{-2mm}
	\caption{The loss changes during the model training process of MIB (backdooring loss) and SMS (primary and self-supervised model seeding loss) } 
	\label{training_opt_process} 
\end{figure}

\section{Further Discussion}

From previous experiments and analysis, although the SMS method provides a framework for unlearning verification for benign samples, the seed embedding introduces more computational costs than backdoor-based unlearning methods \cite{hu2022membership}. As shown in \Cref{tab_total}, the extra computational costs rise by roughly 10–20\%. These costs are acceptable for one‑off model releases but can be prohibitive for streaming or continually retrained services. Future work could amortise the cost by re‑using early‑layer parameters between the decoder and the encoder \cite{pham2018efficient,wallingford2022task}.

Moreover, \citeauthor{zhang2024verification} \cite{zhang2024verification} demonstrate that an untrusted server can mount targeted adversarial attacks that erase the backdoor or seed sub‑space while leaving the rest of the user sample intact, thereby fooling any single‑signal verifier. Unlike backdoor triggers and modified labels, seeds of SMS will not influence the main task and perform like benign features. Hence, simulating the adversarial unlearning for seed-embedded data will ensure a similar effect as unlearning. Techniques like steganography \cite{baluja2019hiding,wu2014steganography} and texture embedding \cite{oechsle2019texture} may strengthen the linkage of seeds and normal feature and improve the concealment of seeds.

\section{Summary and Future Work} \label{summary}

We propose an SMS scheme to verify if users' data has been deleted after executing unlearning processes. To ensure data removal verification effectiveness, seeds must be integrated as inherent features of the unlearned data. We create a self-supervised model seeding task to learn the seed from the samples and propose a model seeding joint training structure to optimize both the primary and self-supervised model seeding tasks simultaneously. After embedding the seeded data, we train a verifying model to identify the seed information from the output of the self-supervised model seeding task. 
Based on extensive experiments, we demonstrate the superiority of SMS in supporting genuine data unlearning verification in both exact and approximate unlearning scenarios. 


As machine unlearning becomes increasingly important, unlearning verification is essential but remains under-explored. The SMS method proposed in this paper serves as a crucial step toward constructing effective unlearning verification for genuine samples in both exact and approximate unlearning scenarios, a goal difficult to achieve with existing backdoor-based verification methods. {However, the SMS also introduces additional computational costs at the same time and may face the potential adversarial threats. Future research could follow this way to develop more efficient and robust unlearning verification approaches.}





\ifCLASSOPTIONcaptionsoff
\newpage
\fi



%


\footnotesize
\bibliographystyle{IEEEtranN}
\bibliography{UAL}

%
\begin{IEEEbiography}[{\includegraphics[width=1in,height=1.25in,clip,keepaspectratio]{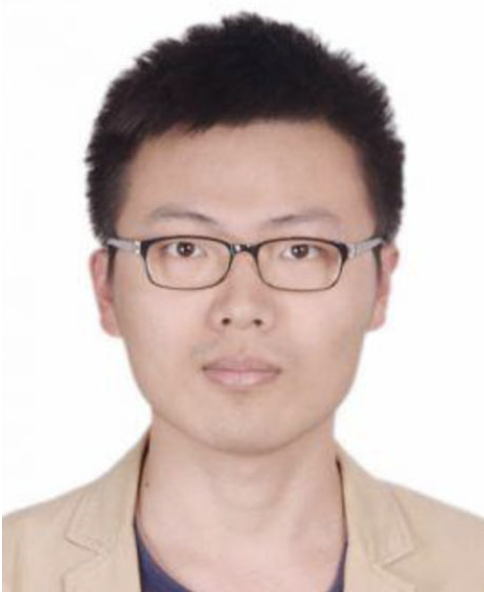}}]{Weiqi Wang}(IEEE M'24)
	received his Ph.D. degree from University of Technology Sydney, Australia, in 2024. He currently is a postdoctoral research associate of University of Technology Sydney, Australia, advised by Prof. Shui Yu. He previously worked as a senior algorithm engineer at the Department of AI-Strategy, Local consumer services segment, Alibaba Group. He has been actively involved in the research community by serving as a reviewer for prestige journals such as ACM Computing Surveys, IEEE Communications Surveys and Tutorials, IEEE TIFS, IEEE TDSC, IEEE TIP, IEEE TMC, IEEE Transactions on SMC, and IEEE IOTJ, and international conferences such as The ACM Web Conference (WWW), ICLR, IEEE ICC, and IEEE GLOBECOM. His research interests are the security and privacy in machine learning.
\end{IEEEbiography}


\begin{IEEEbiography}[{\includegraphics[width=1in,height=1.25in,clip,keepaspectratio]{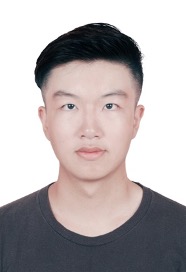}}]{Chenhan Zhang} (IEEE S'19 M'24) obtained his Ph.D. from University of Technology Sydney, Australia, in 2024, where he was advised by Prof. Shui Yu. Before that, he obtained his B.Eng. (Honours) from University of Wollongong, Australia, and  M.S. from City University of Hong Kong, Hong Kong, in 2017 and 2019, respectively. He is currently a postdoctoral research fellow at University of Technology Sydney. His research interests include security and privacy in graph neural networks and trustworthy spatiotemporal cyber physical systems. He has been actively involved in the research community by serving as a reviewer for prestige venues such as ICLR, IJCAI, INFOCOM, IEEE TDSC, IEEE IoTJ, ACM Computing Survey, and IEEE Communications Surveys and Tutorials.
\end{IEEEbiography}

\begin{IEEEbiography}[{\includegraphics[width=1in,height=1.25in,clip,keepaspectratio]{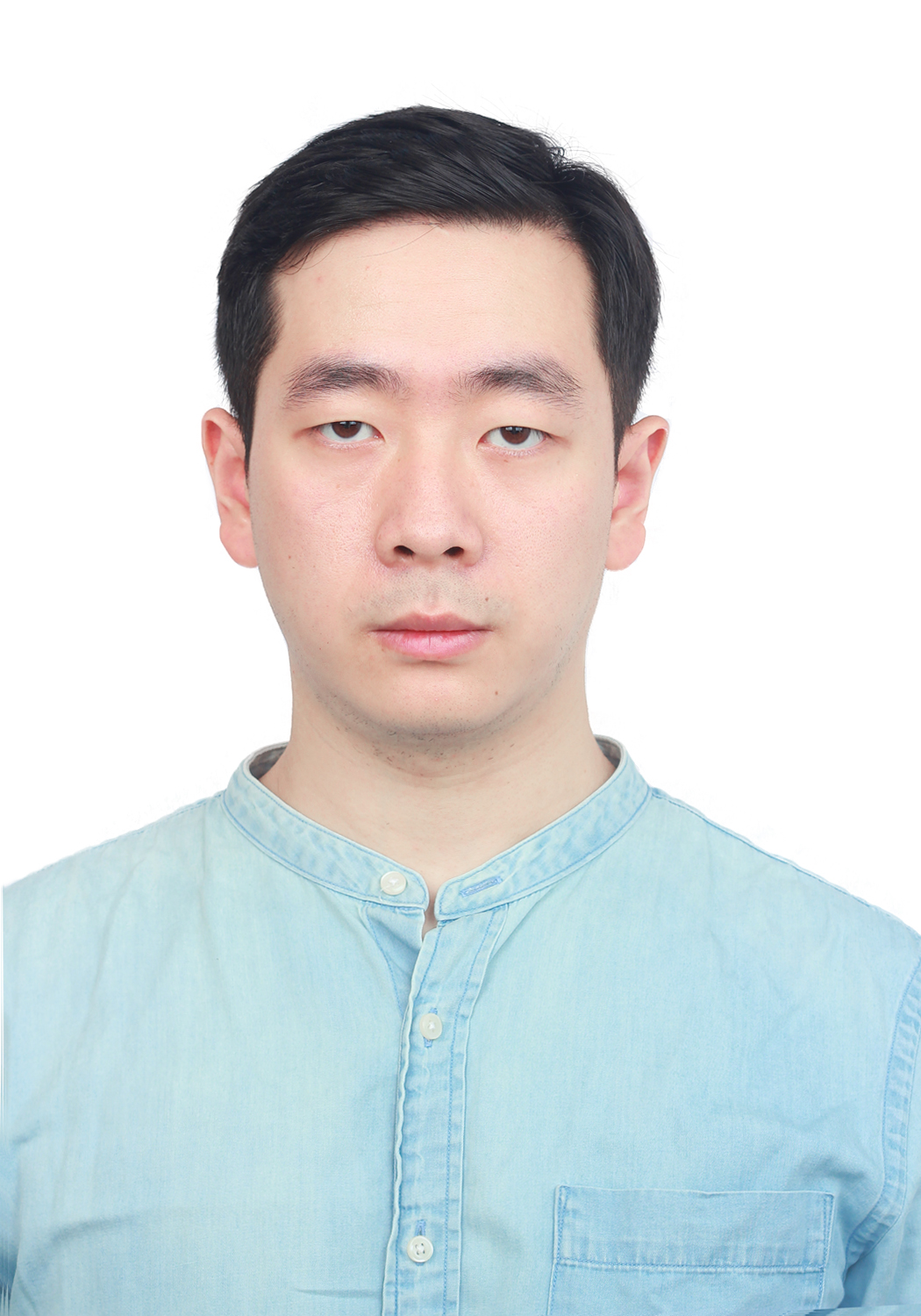}}]{Zhiyi Tian} (IEEE M'24) received the B.S. degree and the M.S. degree from Sichuan University, China, in 2017 and 2020, respectively. He received the Ph.D. degree in 2024 from University of Technology Sydney, Australia. He currently is a research associate of University of Technology Sydney, Australia. His research interests include security and privacy in deep learning, semantic communications. He has been actively involved in the research community by serving as a reviewer for prestige journals, such as ACM Computing Surveys, IEEE Communications Surveys and Tutorials, TIFS, TKDD, and international conferences, such as IEEE ICC and IEEE GLOBECOM.
\end{IEEEbiography}

\begin{IEEEbiography}[{\includegraphics[width=1in,height=1.25in,clip,keepaspectratio]{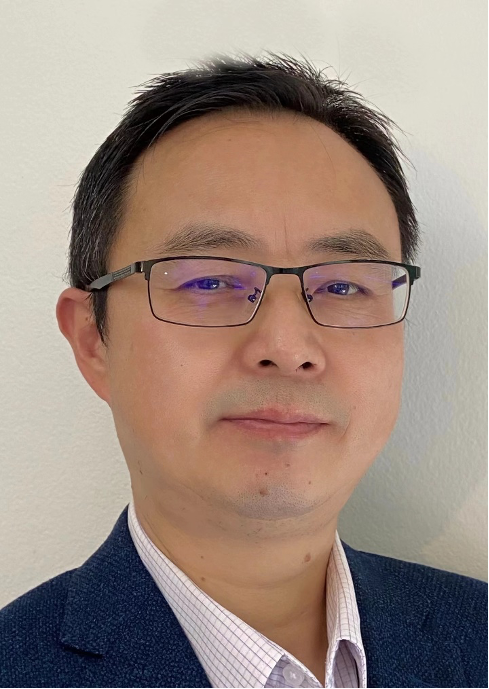}}]{Shui Yu} (IEEE F'23) obtained his PhD from Deakin University, Australia, in 2004. He is a Professor of School of Computer Science, Deputy Chair of University Research Committee, University of Technology Sydney, Australia. His research interest includes Mathematical AI, Cybersecurity, Network Science, and Big Data. He has published seven monographs and edited two books, more than 650 technical papers at different venues. His current h-index is 86. Professor Yu promoted the research field of networking for big data since 2013, and his research outputs have been widely adopted by industrial systems, such as Amazon cloud security. He is currently serving the editorial boards of IEEE Communications Surveys and Tutorials (Area Editor), IEEE Transactions on Cognitive Communications and Networking, and IEEE Transactions on Dependable and Secure Computing. He is a Distinguished Visitor of IEEE Computer Society, and an elected member of the Board of Governors of IEEE Communications Society. He is a member of ACM and AAAS, and a Fellow of IEEE.
\end{IEEEbiography}

\end{document}